\documentclass{article}

\PassOptionsToPackage{numbers}{natbib}

 \usepackage[preprint]{neurips_2026}

\usepackage[utf8]{inputenc} %
\usepackage[T1]{fontenc}    %
\usepackage{hyperref}       %
\usepackage{url}            %
\usepackage{booktabs}       %
\usepackage{amsfonts}       %
\usepackage{nicefrac}       %
\usepackage{microtype}      %
\usepackage{xcolor}         %

\usepackage{microtype}
\usepackage{graphicx}
\usepackage{subcaption}
\usepackage{booktabs}
\usepackage{amsmath}
\usepackage{amssymb}
\usepackage{mathtools}
\usepackage{amsthm}
\usepackage{listings}
\usepackage{enumitem}
\usepackage{nicematrix}
\usepackage{wrapfig}
\usepackage{graphicx}
\usepackage{amsthm}
\newtheorem{theorem}{Theorem}

\usepackage{etoolbox}
\newtoggle{comment}
\toggletrue{comment}

\usepackage[capitalize,noabbrev]{cleveref}

\usepackage{amsmath,amsfonts,bm}

\def\eqref#1{equation~\ref{#1}}

\def\1{\bm{1}}

\def\vf{{\bm{f}}}
\def\vg{{\bm{g}}}
\def\vh{{\bm{h}}}

\def\vl{{\bm{l}}}

\def\vr{{\bm{r}}}
\def\vs{{\bm{s}}}

\def\vx{{\bm{x}}}
\def\vy{{\bm{y}}}
\def\vz{{\bm{z}}}

\def\mA{{\bm{A}}}
\def\mB{{\bm{B}}}
\def\mC{{\bm{C}}}
\def\mD{{\bm{D}}}

\def\mG{{\bm{G}}}

\def\mJ{{\bm{J}}}

\def\mO{{\bm{O}}}

\def\mU{{\bm{U}}}

\def\mW{{\bm{W}}}

\def\mZ{{\bm{Z}}}

\DeclareMathAlphabet{\mathsfit}{\encodingdefault}{\sfdefault}{m}{sl}
\SetMathAlphabet{\mathsfit}{bold}{\encodingdefault}{\sfdefault}{bx}{n}

\renewcommand{\cite}{\citep}

\definecolor{codegreen}{rgb}{0,0.6,0}
\definecolor{codegray}{rgb}{0.5,0.5,0.5}
\definecolor{codepurple}{rgb}{0.07,0,0.53}
\definecolor{codered}{RGB}{189,41,0}
\definecolor{codecomment}{RGB}{153,153,153}
\definecolor{backcolour}{rgb}{0.96,0.96,0.96}
\definecolor{mygreen}{rgb}{0.0, 0.5, 0.0}
\definecolor{royalblue}{rgb}{0.0, 0.14, 0.4}
\definecolor{egyptianblue}{rgb}{0.06, 0.2, 0.65}
\definecolor{royalazure}{rgb}{0.0, 0.22, 0.66}
\definecolor{portlandorange}{rgb}{1.0, 0.35, 0.21}
\definecolor{saddlebrown}{RGB}{139,69,19}
\definecolor{sienna}{RGB}{183,105,68}
\definecolor{saddlebrown}{RGB}{139,69,19}

\hypersetup{
    colorlinks=true,
    linkcolor=sienna,
    citecolor=egyptianblue,
}

\lstdefinestyle{mystyle}{
    backgroundcolor=\color{backcolour},   
    commentstyle=\color{codegreen},
    keywordstyle=\color{codered},
    numberstyle=\tiny\color{codegray},
    stringstyle=\color{codepurple},
    emph={fp16, unroll_steps, problems, dependencies, gradient_accumulation},          
    emphstyle=\color{codered},    
    basicstyle=\ttfamily\footnotesize,
    breakatwhitespace=false,         
    breaklines=true,                 
    captionpos=b,                    
    keepspaces=true,                 
    numbers=left,                    
    numbersep=5pt,                  
    showspaces=false,                
    showstringspaces=false,
    showtabs=false,   
    morekeywords={>,<,.,;,-,!,=,~},
    tabsize=2
}
\title{CODA: Rewriting Transformer Blocks as GEMM-Epilogue Programs}

\author{
\textbf{Han Guo}$^{1}$\thanks{Part of this work was done at Together AI.} \quad
\textbf{Jack Zhang}$^{2}$ \quad
\textbf{Arjun Menon}$^{2}$ \vspace{2mm} \\
\textbf{Driss Guessous}$^{4}$ \quad
\textbf{Vijay Thakkar}$^{4}$ \quad
\textbf{Yoon Kim}$^{1}$ \quad
\textbf{Tri Dao}$^{2,3}$ \vspace{2mm} \\
$^{1}$Massachusetts Institute of Technology \quad
$^{2}$Princeton University \quad
$^{3}$Together AI \quad
$^{4}$Meta  
\vspace{2mm} \\ 
\texttt{hanguo@mit.edu}
}

\begin{document}

\maketitle

\begin{abstract}
Transformer training systems are built around dense linear algebra, yet a nontrivial fraction of end-to-end time is spent on surrounding memory-bound operators. Normalization, activations, residual updates, reductions, and related computations repeatedly move large intermediate tensors through global memory while performing little arithmetic, making data movement an increasingly important bottleneck in otherwise highly optimized training stacks. We introduce \texttt{CODA}, a GPU kernel abstraction that expresses these computations as GEMM-plus-epilogue programs. \texttt{CODA} is based on the observation that many Transformer operators exposed as separate framework kernels can be algebraically reparameterized to execute while a GEMM output tile remains on chip, before it is written to memory. The abstraction fixes the GEMM mainloop and exposes a small set of composable epilogue primitives for scaling, reductions, pairwise transformations, and accumulation. This constrained interface preserves the performance structure of expert-written GEMMs while remaining expressive enough to cover nearly all non-attention computation in the forward and backward pass of a standard Transformer block. Across representative Transformer workloads, both human- and LLM-authored \texttt{CODA} kernels achieve high performance, suggesting that GEMM-plus-epilogue programming offers a practical path toward combining framework-level productivity with hardware-level efficiency.\footnote{Code available at \url{https://github.com/HanGuo97/coda-kernels}.}
\end{abstract}

\section{Introduction}

\begin{wrapfigure}{r}{0.39\textwidth}
    \centering
    \vspace{-27pt}
    \includegraphics[width=0.99\linewidth]{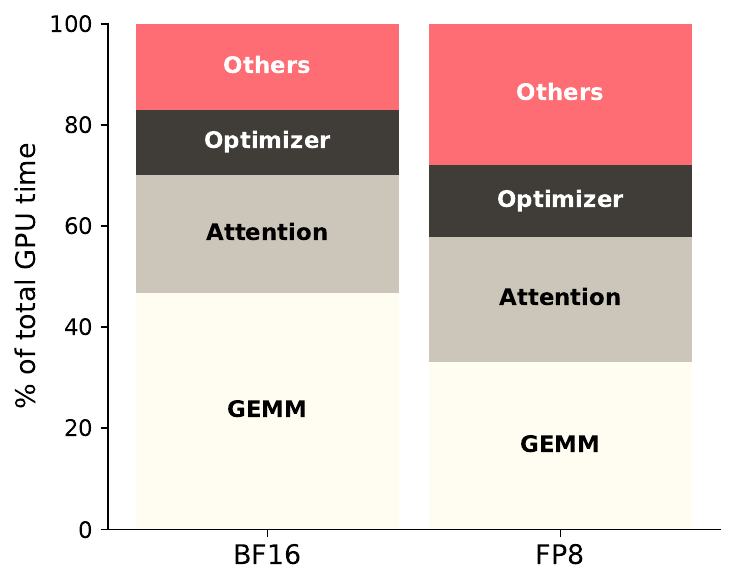}
    \vspace{-17pt}
    \caption{Runtime breakdown for LLaMA-3-style 1B model training on a single H100 using TorchTitan.}
    \vspace{-13pt}
    \label{fig:runtime}
\end{wrapfigure}
LLM training has become just as much of a systems problem as a modeling one.
FLOPs in modern Transformer-based LLMs are dominated by matrix multiplications and attention, whose kernels have been heavily optimized for Tensor Core execution. Yet Transformers, and deep learning architectures more broadly, also contain normalization, activations, residual updates, reductions, and other bandwidth-limited operations that move large tensors through memory while doing little arithmetic. Prior work has shown that data movement is a central bottleneck in Transformer training~\cite{ivanov2021data}; as \Cref{fig:runtime} shows, when training a LLaMA-3-style 1B model on a single H100 using TorchTitan~\cite{liang2024torchtitan}, these non-GEMM operations account for a nontrivial fraction of end-to-end runtime. As hardware increasingly accelerates low-precision matrix multiplication through formats such as FP8 and FP4, this bottleneck is becoming more important, as the cost of materializing intermediate tensors does not improve at the same rate.

Existing programming models make this issue difficult to address. High-level frameworks such as PyTorch express Transformer blocks as operator sequences, with autograd making the backward pass similarly convenient. This is productive, but operator boundaries often become materialization boundaries and obscure fusion opportunities across forward and backward computation. Production-level LLM systems therefore often bypass framework abstractions with hand-written backward passes or custom kernels, as in large-scale LLaMA training~\cite{grattafiori2024llama} and inference systems~\cite{kwon2023efficient,zheng2024sglang}. This work asks whether there is a middle ground. That is, is it possible to  recover much of the performance of custom kernels without giving up the structure needed for programmability and automation?

Our starting observation is that many Transformer computations that appear at the framework level as separate operators can be algebraically reparameterized as GEMM-plus-epilogue programs (\cref{fig:fwd-pass}). In this form, a highly optimized GEMM mainloop produces output tiles, while a programmable epilogue performs tile-local transformations before the result is written to memory (\cref{fig:mainloop-epilogue}). This is efficient on GPUs because the epilogue operates on data that is already produced by the GEMM tile, avoiding additional global-memory round trips for intermediate tensors. With modern pipelined schedules, this epilogue work can often be placed in the shadow of other tiles' mainloops, as in Hopper Ping-Pong GEMM and Blackwell TMEM-based pipelines.
Thus, we extend the epilogue beyond a place for simple post-processing such as scaling or bias addition, elevating it into a structured interface for fusing memory-bound computation into the lifetime of a GEMM tile.

Based on the above, we introduce \texttt{CODA}, a kernel abstraction prototype that realizes this interface. \texttt{CODA} keeps the GEMM mainloop fixed and exposes a small set of composable epilogue primitives for scaling, reductions, pairwise transformations, and accumulation. This programming model is deliberately constrained and yet expressive, as after reparameterization, these primitives cover nearly the entire forward and backward pass of a standard Transformer model while preserving efficiency. \texttt{CODA} inserts computation into the epilogue of a known high-performance GEMM before intermediate tensors are materialized in global memory, capturing a broad class of memory-bound computations surrounding dense linear algebra. Transformers are our primary application, but the same GEMM-plus-epilogue view applies more broadly whenever high-throughput matrix multiplication is surrounded by tile-expressible, data-movement-bound computations.

Finally, this structure makes automation more practical. Epilogue fusion is already established in high-performance GEMM libraries, but applying it to Transformer workloads remains a low-level engineering task. \texttt{CODA} targets this gap by providing Transformer-specific epilogue primitives on top of a tuned GEMM mainloop. Human or LLM-based authors can assemble these primitives into reparameterized Transformer kernels rather than synthesizing arbitrary CUDA. Across representative workloads, both authoring modes achieve high performance, suggesting that domain-specific epilogue abstractions can make established GEMM fusion techniques more programmable for LLM kernels.

\begin{figure*}[t]
    \centering
    \vspace{-15pt}
    \includegraphics[width=0.99\textwidth]{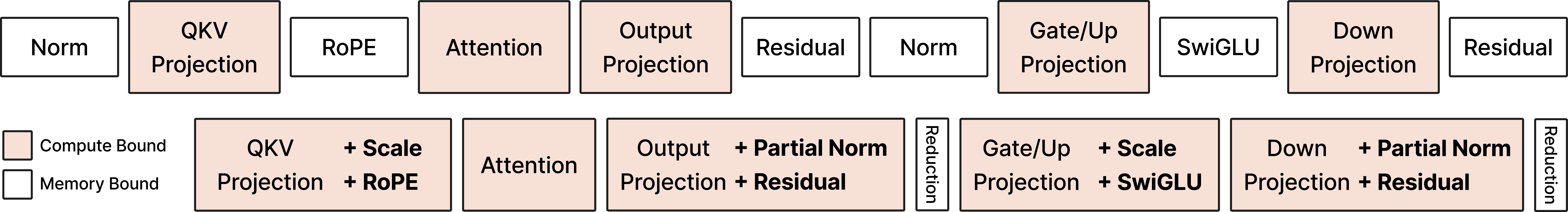}
    \caption{Forward pass of a standard Transformer layer. The top row shows the canonical formulation, which maps to a mix of compute- and memory-bound kernels. We reparameterize the computation so that most memory-bound operations are subsumed into the epilogues of compute-bound kernels.}
    \vspace{-17pt}
    \label{fig:fwd-pass}
\end{figure*}

\section{Background and Related Works}

\subsection{Programming Models for LLM Systems}
\label{sec:background-programming-model}

Modern LLM systems are programmed at multiple abstraction levels. Frameworks such as PyTorch and JAX express models as tensor-operator graphs and integrate naturally with automatic differentiation, but operator boundaries often become materialization ones.

Compiler systems lower tensor programs to optimized kernels through graph rewriting, scheduling, code generation, and autotuning~\cite{ansel2024pytorch,chen2018tvm,tillet2019triton}. Algebraic reformulation is another important source of performance, as shown by TASO~\cite{jia2019taso} and Mirage~\cite{wu2025mirage}. However, rapidly evolving accelerators make peak performance a moving target for general-purpose compilers.

Closer to the hardware level, programmers use kernel DSLs and libraries such as Triton~\cite{tillet2019triton}, ThunderKittens~\cite{spector2024thunderkittens,spector2025look,sul2025parallelkittens}, TileLang~\cite{wang2025tilelang}, CuTeDSL~\cite{Thakkar_CUTLASS_2023}, Gluon, and TLX, or rely on specialized LLM kernels in vLLM~\cite{kwon2023efficient}, SGLang~\cite{zheng2024sglang}, FlashInfer~\cite{ye2025flashinfer}, and Liger Kernels~\cite{hsu2024liger}. These approaches deliver high performance, but extending them to new transformations or backward computations still requires substantial low-level engineering.

\subsection{GEMM Mainloops and Epilogue Fusion}
\label{sec:background-gemm-epilogue}

Matrix multiplication is the central compute primitive in modern LLM workloads. A high-performance GEMM kernel is typically divided into a mainloop and an epilogue. The mainloop performs the tiled matrix multiply-accumulate computation, while the epilogue transforms the computed output tile and efficiently writes it back to global memory.

\begin{wrapfigure}{r}{0.39\textwidth}
    \centering
    \vspace{-25pt}
    \includegraphics[width=\linewidth]{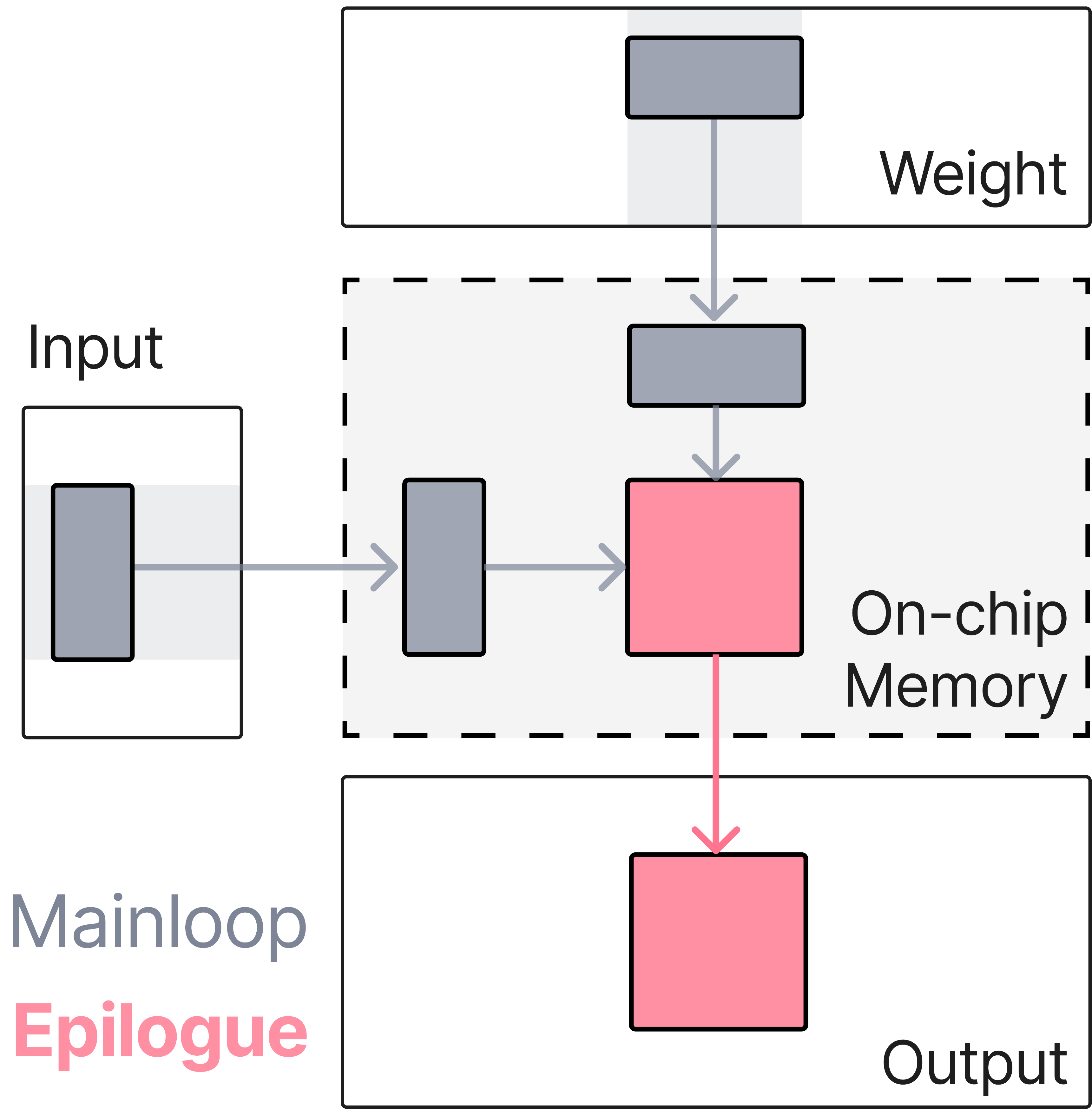}
    \vspace{-12pt}
    \caption{A GEMM mainloop computes output tiles; the epilogue transforms each tile before the final global-memory store.}
    \label{fig:mainloop-epilogue}
    \vspace{-12pt}
\end{wrapfigure}

The epilogue is a natural place to implement fusions because the output of the matmul is already present on chip close to compute cores. Practical epilogues commonly perform scaling, bias addition, activations, residual updates, data type conversions, tile-wise reductions and other output elementwise operations, avoiding separate kernel launches and extra global-memory round trips. Modern kernel libraries formalize this separation directly: CUTLASS~\cite{Thakkar_CUTLASS_2023} represents GEMM kernels as a composition of a collective mainloop and a collective epilogue, while Epilogue Visitor Trees further express epilogues as compositions of primitives~\cite{chen2024evt}.

This flexibility operates under a locality constraint. An epilogue sees only the local output tile, its accumulators, and consistently indexed auxiliary tensors, meaning that operations requiring global reductions or cross-tile communication must be reformulated into tile-local pieces or handled in a separate pass. CODA builds on this interface, keeping the high-performance GEMM mainloop fixed and using the epilogue as a programmable site for nearby memory-bound computation.

\section{CODA}
\label{sec:coda}

The previous section argued that GEMM epilogues are a natural place to fuse memory-bound computation into dense linear algebra. We now describe \texttt{CODA}, a GPU kernel abstraction that realizes this idea. \Cref{sec:epilogue-primitives} identifies a small set of epilogue primitives that map efficiently to GPU execution. \Cref{sec:reparameterization} shows how the non-attention and non-embedding portions of the Transformer forward and backward pass can be reparameterized using these primitives. Finally, \Cref{sec:implementations} describes their implementation and our LLM-oriented authoring workflow.

\subsection{Efficient Epilogue Primitives}
\label{sec:epilogue-primitives}

\texttt{CODA} programs the GEMM epilogue while keeping the mainloop fixed and highly optimized. For each output tile, an epilogue may load auxiliary data, transform accumulator values, emit auxiliary results, and store the final output. This interface is deliberately restricted to tile-local computation rather than arbitrary global communication. Our epilogue template, shown in \Cref{label:evt-template}, is inspired by Epilogue Visitor Trees~\cite{chen2024evt}.
\texttt{CODA} provides five classes of epilogue primitives:
\begin{enumerate}[leftmargin=*,nosep,nolistsep,noitemsep]
    \item \emph{Elementwise and pairwise maps:} apply local transformations to accumulator values, including residual updates, activations, RoPE-style rotations, and SwiGLU-style gates.
    \item \emph{Vector (Rank-1 Tensor) loads and stores:} load row or column vectors, broadcast them over an output tile, and optionally write vector-valued auxiliary results.
    \item \emph{Tile (Rank-2 Tensor) loads and stores:} load or store matrix tiles, such as residual streams, saved activations, or intermediate values needed by the backward pass.
    \item \emph{Tile (Rank-2 Tensor) reductions:} compute partial reductions over rows or columns of an output tile, to be combined later by a lightweight auxiliary kernel.
    \item \emph{Stateful transforms:} maintain running tile state, such as the max and sum-exp statistics used in online log-sum-exp and cross-entropy.
\end{enumerate}

These primitives are intentionally narrow, operating at a level low enough to compile to efficient epilogue code and expressive enough to capture the memory-bound operations surrounding GEMMs in our Transformer reparameterizations, as shown next.

\subsection{Reparameterizing Transformers as Epilogues}
\label{sec:reparameterization}

We now show that the primitive set above is sufficient for much of Transformer computation. After lightweight algebraic reparameterization, many non-attention and non-embedding components of a standard Transformer forward pass can be written as
\begin{align*}
    \text{GEMM:}\quad
    \vh = \vx \mW,
    \qquad
    \text{Epilogue:}\quad
    \vy[i,j] = \vf[i,j]\!\left(\vh[i,j]\right),
\end{align*}
where $[i,j]$ indexes an output tile and $\vf[i,j]$ is the tile function implemented in the GEMM epilogue. The epilogue is either fully tile-local, or tile-local up to partial results that are combined by a lightweight auxiliary reduction. We first apply this view to the forward pass, then show that independent tile functions preserve the same GEMM-epilogue structure in the backward pass.

\subsubsection{GEMM-Residual-RMSNorm-GEMM Pattern}
\label{sec:gemm-residual-rmsnorm-gemm}

A recurring pattern in pre-normalized Transformers is a GEMM followed by a residual update and normalization, then another GEMM. This pattern appears across several adjacent sublayers:
\begin{enumerate}[leftmargin=*,nosep,nolistsep,noitemsep]
    \item attention output projection $\rightarrow$ residual stream $\rightarrow$ RMSNorm $\rightarrow$ MLP gate/up projection;
    \item MLP down projection $\rightarrow$ residual stream $\rightarrow$ RMSNorm $\rightarrow$ attention QKV projection;
    \item final MLP down projection $\rightarrow$ residual stream $\rightarrow$ final RMSNorm $\rightarrow$ language modeling head.
\end{enumerate}
Although these cases are usually written as parts of different modules, they share the same computational structure:
\begin{align*}
\vy
= \operatorname{RMSNorm}(\vx \mW_0 + \vz, \boldsymbol{\gamma}) \mW_1
= \Bigl(r \, \bigl(\vx \mW_0 + \vz\bigr) \odot \boldsymbol{\gamma}\Bigr) \mW_1,
\label{eq:gemm-res-rms-gemm}
\end{align*}
where $\vz$ denotes the residual stream and
$r = 1 / \operatorname{rms}(\vx \mW_0 + \vz)$ is the row-wise inverse RMS factor. This pattern crosses the usual module boundary: it couples the output projection of one sublayer with the input projection of the next.

Residual addition and multiplication by the RMSNorm weight $\boldsymbol{\gamma}$ are tile-local, so they can be fused into a GEMM epilogue. The row-wise factor $r$, however, requires a reduction across the hidden dimension, which is larger than a single output tile. In the canonical computation, $r$ is applied before the second GEMM, creating an apparent dependency between normalization and the next projection.

\begin{wrapfigure}{r}{0.6\textwidth}
  \centering
  \vspace{-17pt}
  \includegraphics[width=0.99\linewidth]{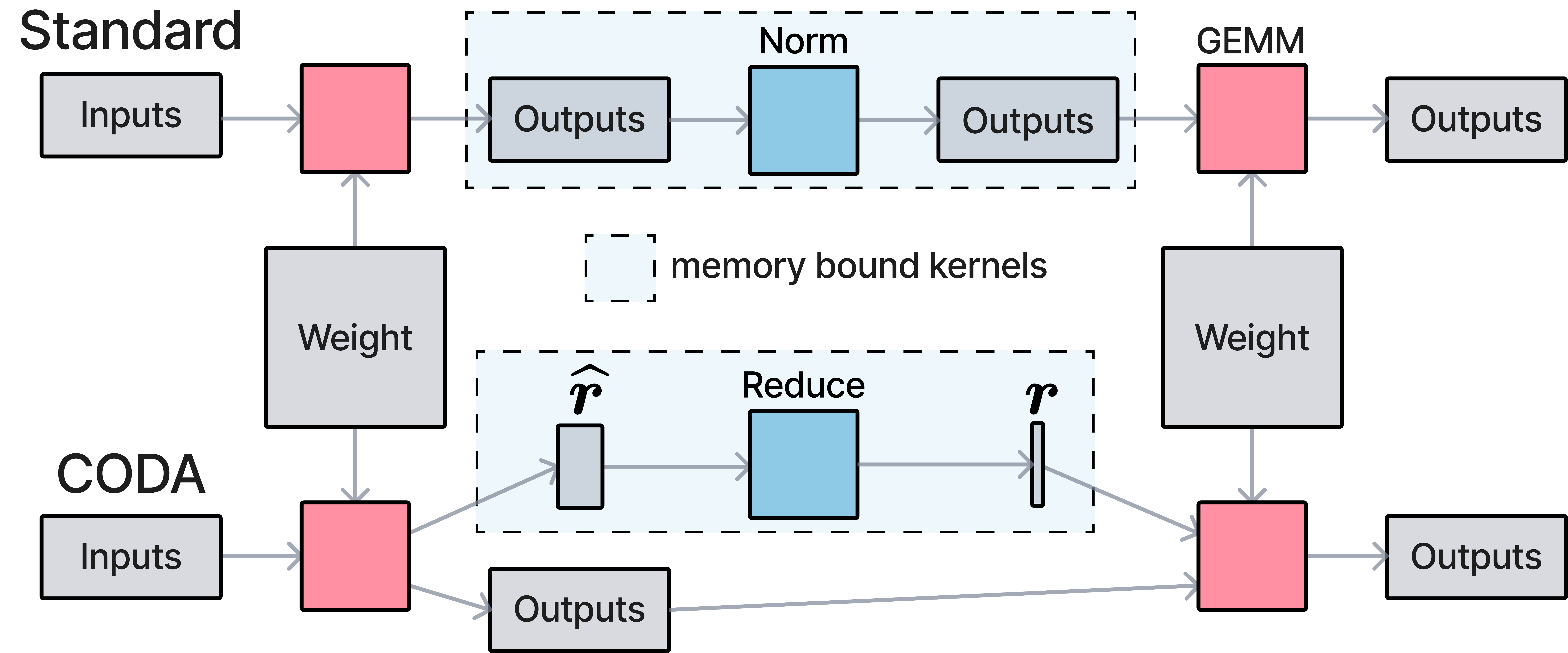}
  \vspace{-1pt}
  \caption{GEMM-RMSNorm-GEMM reparameterization.}
  \label{fig:gemm-res-rms-gemm-kernels}
  \vspace{-11pt}
\end{wrapfigure}

We address the reduction by splitting it into two levels. The first GEMM epilogue computes tile-local partial reductions, and a small auxiliary kernel reduces these partials across tiles to obtain $r$. Since the auxiliary kernel reads a few partial values per tile rather than the full activation tensor, its memory traffic is much smaller than that of a standalone RMSNorm.

The apparent dependency on $r$ can be removed algebraically. Since $r$ is shared across the row, it commutes with the second GEMM:
\begin{align*}
\vy
= \Bigl(r \, \bigl(\vx \mW_0 + \vz\bigr) \odot \boldsymbol{\gamma}\Bigr) \mW_1
= r \, \Bigl(\bigl(\vx \mW_0 + \vz\bigr) \odot \boldsymbol{\gamma}\Bigr) \mW_1 .
\end{align*}
Thus, the row-wise scale does not need to be applied before the second GEMM. It can instead be delayed to the epilogue of the second GEMM, after the projection has been computed.

Concretely, the computation decomposes into two GEMMs and one lightweight reduction (\cref{fig:gemm-res-rms-gemm-kernels}):
\begin{align*}
&\text{GEMM 1:}& \vh_0 &= \vx \mW_0, \quad
&\text{Epilogue 1:}& \quad
&\vh_1[i,j] &= \vh_0[i,j] + \vz[i,j], \\
&&&&&&\vh_2[i,j] &= \vh_1[i,j] \odot \boldsymbol{\gamma}[j], \\
&&&&&&\widehat{\vr}[i,j] &= \operatorname{partialRMS}(\vh_1[i,j]),\\[7pt]
&\text{GEMM 2:} &\vh_3 &= \vh_2 \mW_1, \qquad
&\text{Epilogue 2:}& \quad
&\vy[i,j] &= r[i] \, \vh_3[i,j].
\end{align*}
\begin{wrapfigure}{r}{0.30\textwidth}
  \centering
  \vspace{-17pt}
  \includegraphics[width=0.99\linewidth]{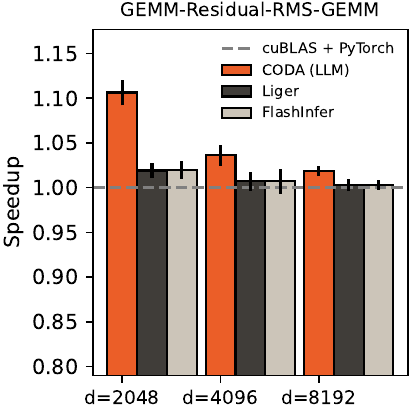}
  \vspace{-15pt}
  \caption{Benchmarks.}
  \label{fig:gemm-res-rms-gemm-results}
  \vspace{-27pt}
\end{wrapfigure}
Here $r = 1 / \sqrt{\operatorname{reduce}(\widehat{\vr}) + \epsilon}$ is computed by a small auxiliary reduction over the tile partials. This decomposition replaces a standalone RMSNorm kernel with tile-local epilogue work around the two GEMMs, plus a lightweight auxiliary reduction.

In \Cref{fig:gemm-res-rms-gemm-results}, we benchmark this reparameterization against existing implementations on LLaMA-style configurations with a batch of 16K tokens. We vary the hidden dimension across representative model scales, with $d \in \{2048, 4096, 8192\}$ corresponding roughly to 1B, 7B, and 70B models, respectively. Our GEMM-Epilogue kernel is generated by an LLM provided with the above abstractions (explained in more detail in \cref{sec:llm-kernel-authoring}).

\begin{wrapfigure}{r}{0.27\textwidth}
    \centering
    \vspace{-20pt}
    \includegraphics[width=\linewidth]{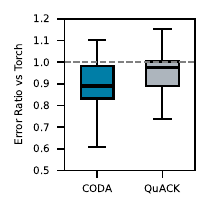}
    \vspace{-22pt}
    \caption{Relative error.}
    \label{fig:gemm-rms-gemm-numerics}
    \vspace{-10pt}
\end{wrapfigure}
\paragraph{Numerics.}
The reparameterization changes where the RMSNorm scale is applied: the row-wise factor $r$ is delayed from before the second GEMM to the second GEMM epilogue. We compare \texttt{BF16} GEMM-RMSNorm-GEMM outputs against an \texttt{FP32} reference on Llama-3 8B layers. We report the errors of \texttt{CODA} and QuACK, on which our GEMM template is based, normalized by the error of the standard PyTorch path. \Cref{fig:gemm-rms-gemm-numerics} suggests that a more accurate GEMM mainloop can reduce numerical error, and that \texttt{CODA}'s reparameterized epilogue can reduce it further.

\subsubsection{GEMM with Pairwise Activations}
\label{sec:gemm-pairwise-activation}

A second common pattern in Transformers is a GEMM followed by a \emph{pairwise} activation. Unlike an elementwise activation, which transforms each feature independently, a pairwise activation consumes two adjacent feature values and produces one or two outputs. 
\begin{align*}
\vh = \vx \mW,
\qquad
\vh_a[i,j], \vh_b[i,j] = \operatorname{split}\left(\vh[i,j]\right),
\qquad
\vy[i,j] = \vf[i,j]\left(\vh_a[i,j], \vh_b[i,j]\right).
\end{align*}
\begin{wrapfigure}{r}{0.375\textwidth}
  \centering
  \vspace{-15pt}
  \includegraphics[width=\linewidth]{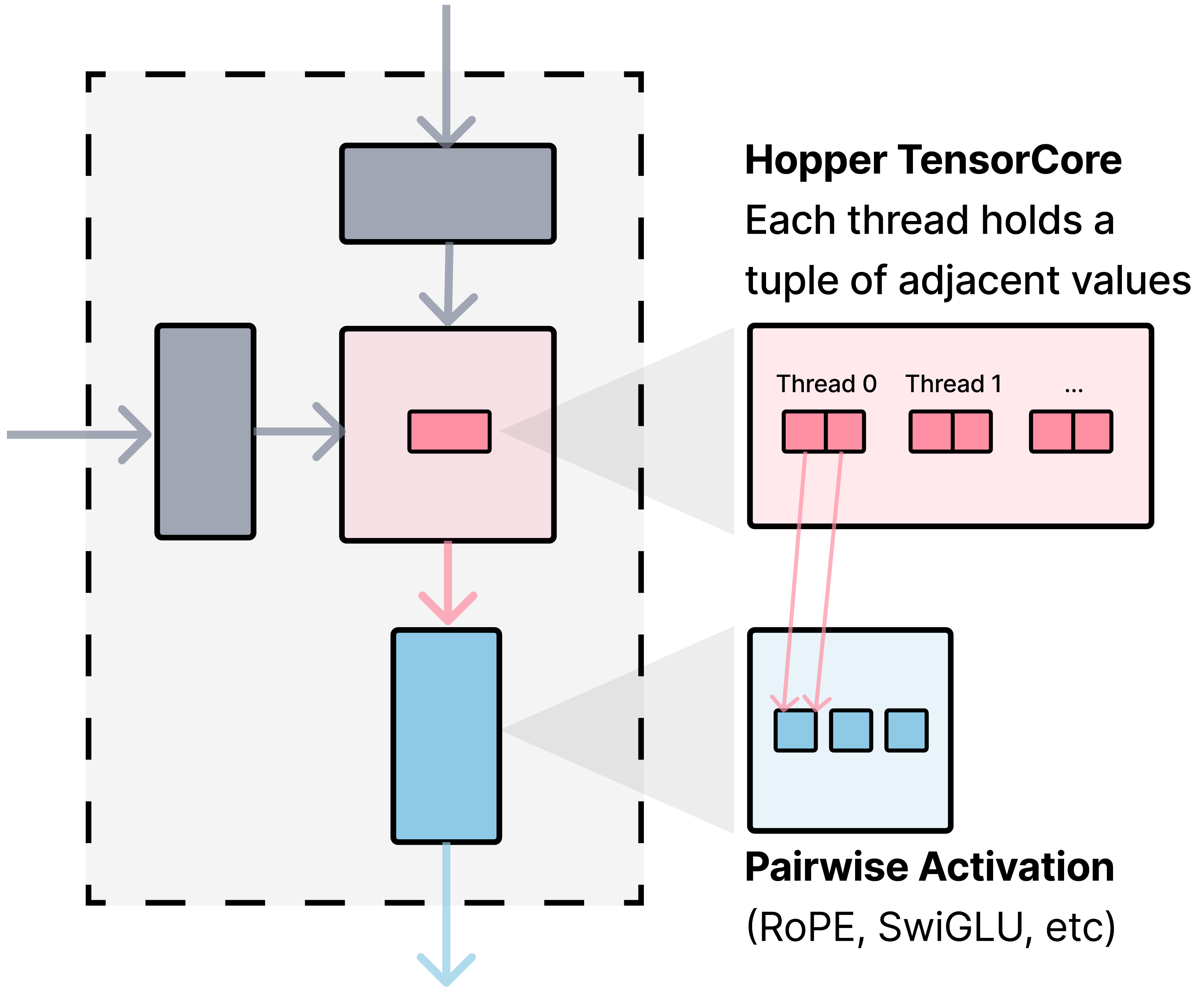}
  \vspace{-12pt}
  \caption{Pairwise activations operate on local feature pairs in the GEMM epilogue.}
  \label{fig:kernel-1}
  \vspace{-27pt}
\end{wrapfigure}
This form captures several operations in Transformer blocks:
\begin{itemize}[leftmargin=*,nosep,nolistsep,noitemsep]
    \item RoPE rotates each feature pair and return two outputs;
    \item SwiGLU combines gate and value stream into one output;
    \item SwiGLU backward pass maps one incoming gradient into gradients for both paired inputs.
\end{itemize}
Pairwise activations couple neighboring feature lanes and may change the feature dimension. A naive implementation materializes the GEMM output, splits it into paires, and applies the activation in a separate kernel. This adds memory traffic and sometimes materializes an expanded intermediate, as in SwiGLU.

Instead, we arrange paired features to be adjacent along the output-feature dimension. This matches the Hopper Tensor Core accumulator layout exposed to the epilogue, where each thread holds a small tuple of adjacent output values in registers before they are stored. The epilogue can therefore apply $f$ directly to each pair with register-level computation.

This removes the standalone activation kernel and avoids materializing the paired intermediate in global memory. The same idea applies to dimension-preserving operations such as RoPE, dimension-reducing operations such as SwiGLU, and dimension-expanding operations in the backward pass, as long as the pairing is reflected in the GEMM output layout. See~\Cref{fig:gemm-epilogue} for performance benchmarks.
\begin{figure}[t]
    \centering
    \vspace{-15pt}
    \includegraphics[width=1.0\textwidth]{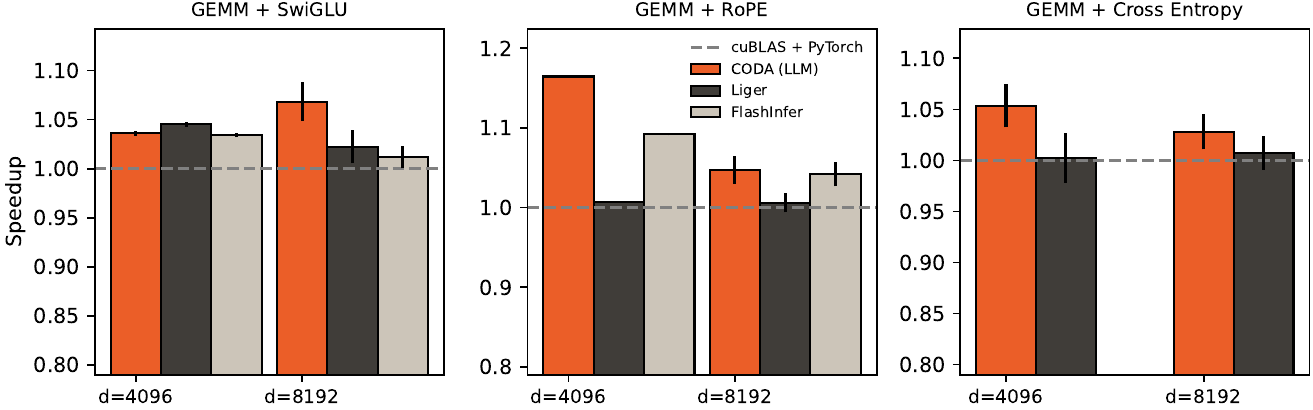}
    \vspace{-15pt}
    \caption{
    Kernel-level speedups for representative GEMM-plus-epilogue primitives across $MNK$ sizes. RoPE uses an output dimension of $3N$ for QKV projections, and cross-entropy uses a $32\mathrm{K}$ vocabulary. Speedups are relative to cuBLAS with \texttt{torch.compile}.
    }
    \label{fig:gemm-epilogue}
    \vspace{-10pt}
\end{figure}

\subsubsection{GEMM with Cross-Entropy Loss}
\label{sec:gemm-cross-entropy}

Cross-entropy loss can also be expressed as a GEMM with epilogue-side reductions, as shown by Cut Cross-Entropy~\cite{wijmans2025cut}. Let $\vh_i = \vx_i \mW_{\mathrm{lm}}$ be the logits for token $i$, and let $\vy_i$ be its target label. The per-token loss is
\begin{align*}
\ell_i
= -\vh_{i,\vy_i} + \log \sum_k \exp(\vh_{i,k}).
\end{align*}
Thus, the loss decomposes into an indexed logit and a row-wise log-sum-exp over vocabulary entries.

Both terms fit the GEMM-plus-epilogue pattern. The indexed logit can be selected from the GEMM output tile using the target label, while the LSE can be accumulated as tile-local maximum and sum-exp statistics. A small auxiliary reduction then combines these statistics across tiles, avoiding a standalone memory-bound softmax over the full logits.\footnote{We use a separate final reduction rather than atomics, and materialize logits to simplify the backward pass.} See \Cref{fig:gemm-epilogue} for performance benchmarks.

\subsubsection{Backward Pass}
\label{sec:backward-gemm-residual-rmsnorm-gemm}

The preceding sections show that much of the Transformer forward pass can be reparameterized as GEMMs with epilogues, plus lightweight auxiliary reductions. We now show that the backward pass preserves the same structure.

\paragraph{GEMM with elementwise epilogue.}
Consider two GEMMs separated by an elementwise epilogue:
\begin{align*}
    \vh = \vx \mW_0, \qquad
    \vh^\prime = f(\vh), \qquad
    \vy = \vh^\prime \mW_1 ,
\end{align*}
where $f$ is applied elementwise. Given an upstream gradient $\nabla_{\vy}\mathcal{L}$, reverse-mode differentiation gives
\begin{align*}
    \nabla_{\vh^\prime}\mathcal{L}
    = \nabla_{\vy}\mathcal{L} \mW_1^\top, \qquad
    \nabla_{\vh}\mathcal{L}
    = \nabla_{\vh^\prime}\mathcal{L} \odot f^\prime(\vh), \qquad
    \nabla_{\vx}\mathcal{L}
    = \nabla_{\vh}\mathcal{L} \mW_0^\top .
\end{align*}
Thus, the backward computation has the same structure as the forward computation: GEMM, local transformation, GEMM. The only difference is the direction of fusion. In the forward pass, $f$ is fused into the epilogue of the \textit{preceding} GEMM that produces $\vh$; in the backward pass, multiplication by $f^\prime(\vh)$ is fused into the epilogue of the \textit{following} GEMM that produces $\nabla_{\vh^\prime}\mathcal{L}$ (\Cref{fig:gemm-epilogue-fwd-bwd}).

\begin{theorem}
\label{thm:gemm-epilogue-backward}
Consider a sequence of GEMM-with-epilogue blocks followed by a final GEMM:
\begin{align*}
\vh_{\ell} &= \vx_{\ell-1}\mW_{\ell}, \qquad
\vx_{\ell}[i,j] =
\vf_{\ell}[i,j]\!\left(\vh_{\ell}[i,j]\right),
\qquad \ell=1,\ldots,L-1, \\
\vh_L &= \vx_{L-1}\mW_L .
\end{align*}
Assume each tile function $\vf_{\ell}[i,j]$ acts only on its corresponding
GEMM output tile. Then the activation gradients can be computed with the same
GEMM-with-epilogue structure:
\begin{align*}
\nabla_{\vx_{\ell-1}}\mathcal{L}
&=
\nabla_{\vh_{\ell}}\mathcal{L}\mW_{\ell}^{\top}, \qquad
\nabla_{\vh_{\ell-1}}\mathcal{L}[i,j]
=
\vg_{\ell-1}[i,j]\!\left(
\nabla_{\vx_{\ell-1}}\mathcal{L}[i,j],\;
\vh_{\ell-1}[i,j]
\right),
\qquad \ell=L,\ldots,2, \\
\nabla_{\vx_0}\mathcal{L}
&=
\nabla_{\vh_1}\mathcal{L}\mW_1^{\top}.
\end{align*}
Here $\vg_{\ell}[i,j]$ is the tile-local backward rule for
$\vf_{\ell}[i,j]$: it maps the gradient of the epilogue output tile
$\vx_{\ell}[i,j]$ to the gradient of the epilogue input tile
$\vh_{\ell}[i,j]$. The weight gradients are GEMMs:
\begin{align*}
\nabla_{\mW_{\ell}}\mathcal{L}
=
\vx_{\ell-1}^{\top}\nabla_{\vh_{\ell}}\mathcal{L},
\qquad \ell=1,\ldots,L .
\end{align*}
Thus, tile-local epilogues in the forward pass induce tile-local epilogues in
the backward pass, while the surrounding linear maps remain GEMMs.
\end{theorem}

\begin{proof}[Proof sketch]
For the GEMM $\vh_{\ell}=\vx_{\ell-1}\mW_{\ell}$, reverse-mode differentiation gives
\begin{align*}
\nabla_{\vx_{\ell-1}}\mathcal{L}
=
\nabla_{\vh_{\ell}}\mathcal{L}\mW_{\ell}^{\top},
\qquad
\nabla_{\mW_{\ell}}\mathcal{L}
=
\vx_{\ell-1}^{\top}\nabla_{\vh_{\ell}}\mathcal{L},
\end{align*}
which are both GEMMs. For the epilogue
$\vx_{\ell}[i,j]=\vf_{\ell}[i,j](\vh_{\ell}[i,j])$, define the local backward
rule by
\begin{align*}
\nabla_{\vh_{\ell}}\mathcal{L}[i,j]
=
\vg_{\ell}[i,j]\!\left(
\nabla_{\vx_{\ell}}\mathcal{L}[i,j],\;
\vh_{\ell}[i,j]
\right).
\end{align*}
Since each $\vf_{\ell}[i,j]$ depends only on its own tile, the corresponding
$\vg_{\ell}[i,j]$ also acts only on that tile. The backward pass therefore
introduces no new cross-tile communication and preserves the
GEMM-with-epilogue structure.
\end{proof}

\begin{figure}[t]
  \centering
  \vspace{-15pt}
  \includegraphics[width=\linewidth]{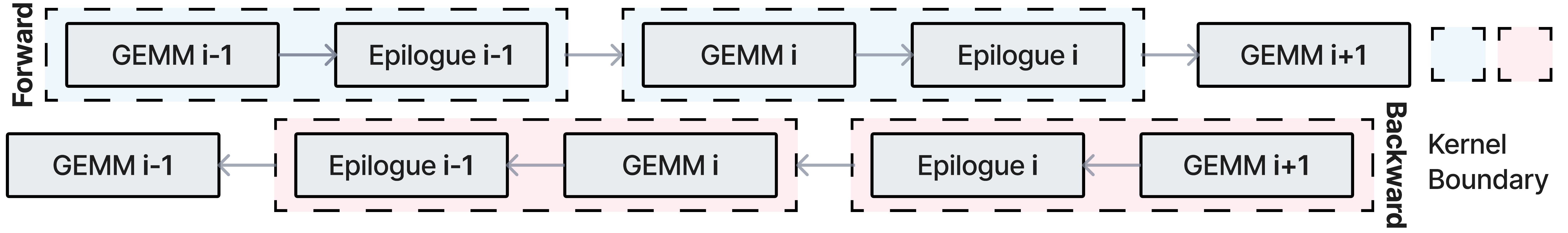}
  \vspace{-17pt}
  \caption{Forward and backward fusion for GEMM--epilogue blocks. Forward epilogues attach to the GEMM that produces their input, while backward epilogues attach to the GEMM that produces the gradient with respect to their output.}
  \label{fig:gemm-epilogue-fwd-bwd}
  \vspace{-10pt}
\end{figure}

\paragraph{RMSNorm backward.}
RMSNorm is the main case where the backward pass is not purely tile-local. Its backward rule introduces two reductions: a row-wise statistic needed for the input gradient, and a feature-wise reduction across rows for the RMSNorm weight gradient. A direct implementation computes both in a standalone RMSNorm backward kernel, requiring additional reads of activation-sized tensors.
However, the row-wise statistic can be moved to a neighboring GEMM boundary. Consider
\begin{align*}
\vh_0 = \vx \mW_0, \qquad
\vh_1 = f(\vh_0), \qquad
\vh_2 = \operatorname{RMSNorm}(\vh_1,\boldsymbol{\gamma}), \qquad
\vy = \vh_2 \mW_1 .
\end{align*}
RMSNorm backward requires the row-wise inner product
\begin{align*}
\vs
=
\frac{1}{d}
\operatorname{sum}_{\mathrm{cols}}
\left(
\nabla_{\vh_2}\mathcal{L} \odot \vh_2
\right).
\end{align*}
Using $\nabla_{\vh_2}\mathcal{L}=\nabla_{\vy}\mathcal{L}\mW_1^\top$ and $\vy=\vh_2\mW_1$, this statistic can be equivalently written as
\begin{align*}
\vs
=
\frac{1}{d}
\operatorname{sum}_{\mathrm{cols}}
\left(
\nabla_{\vy}\mathcal{L} \odot \vy
\right).
\end{align*}
This identity changes where the statistic is computed. Instead of launching a standalone RMSNorm backward kernel to read $\vh_2$ and $\nabla_{\vh_2}\mathcal{L}$, we accumulate the same row-wise quantity at a boundary where $\vy$ and $\nabla_{\vy}\mathcal{L}$ are already available, thereby exposing the computation to epilogue fusion.

In consecutive Transformer patterns, each GEMM epilogue can therefore accumulate the row-wise statistic needed by the preceding RMSNorm backward. The RMSNorm weight gradient is handled similarly by emitting tile partials for the reduction across rows. Overall, RMSNorm backward becomes GEMM--epilogue kernels plus lightweight auxiliary reductions over tile partials. We give the full derivation and kernel organization in \Cref{sec:backward-rmsnorm-epilogue}, with benchmarks in \Cref{fig:block-experiments}.

\subsection{Implementation}
\label{sec:implementations}

We implement \texttt{CODA} on top of CuTeDSL, which provides Python-level kernel authoring while retaining low-level control over details such as layouts and memory movement.

\textbf{Data movement.}
Vector loads handle small broadcast operands such as RMSNorm weights. These values are staged once in shared memory and reused across subtiles. Tile loads handle larger operands such as residual activations, and uses Tensor Memory Accelerator transfers between global memory and shared memory, allowing data movement to overlap with epilogue computation. Stores follow similar tile-granular path for transformed outputs, saved intermediates, and reduction partials.

\textbf{Local computation.}
Pairwise maps act on neighboring features, covering dimension-preserving operations such as RoPE, dimension-reducing operations such as SwiGLU, and dimension-expanding operations in the backward pass. When a map changes the feature dimension, the epilogue performs the corresponding local layout adjustment, such as compacting dimension-reducing outputs or packing pairs of 16-bit values for dimension-expanding outputs.

\textbf{Reductions.}
Tile-wise reductions follow the ownership of GEMM output fragments. Row-wise reductions are accumulated by the warp that owns the row. Column-wise reductions may span multiple warps, so each warp first produces a partial result, and these partials are combined through shared memory.

\subsubsection{LLM-Oriented Authoring}
\label{sec:llm-kernel-authoring}

\texttt{CODA} is designed both for LLM workloads and for LLM-assisted authoring. Rather than asking a model to synthesize arbitrary CUDA or discover a hardware schedule from scratch, \texttt{CODA} exposes a constrained space of epilogue programs around expert-designed GEMM mainloops. Its primitives already encode efficient implementation strategies, so LLM-based authoring becomes a problem of composing vector loads, tile loads, pairwise maps, reductions, and stores for a given Transformer computation. This lightweight use of LLMs is complementary to prior work on kernel generation, which often relies on orchestration, search, execution feedback, or post-training~\cite{ouyang2025kernelbench,su2025cuda,chen2025cuda,lange2025towards,yuksekgonul2026learning}.

Because CuTeDSL is relatively new, current models have limited exposure to its idioms. We therefore provide curated demonstrations for each abstraction. In practice, the repository itself acts as a growing demonstration set, with new kernels being written by adapting and composing existing examples.

\paragraph{Compositions.}
Transformer kernels often combine several epilogue operations, such as residual addition and RMSNorm scaling. Monolithic fused epilogues lead to large, repetitive implementations that are difficult to place in context. \texttt{CODA} instead represents each epilogue as a composition of reusable primitives: the LLM specifies the local epilogue program, while the library supplies the fixed GEMM mainloop and implementation pattern for each primitive. New fused kernels are therefore assembled from reusable building blocks instead of being rewritten from scratch.

\begin{figure}[t]
    \centering
    \vspace{-15pt}
    \includegraphics[width=0.99\textwidth]{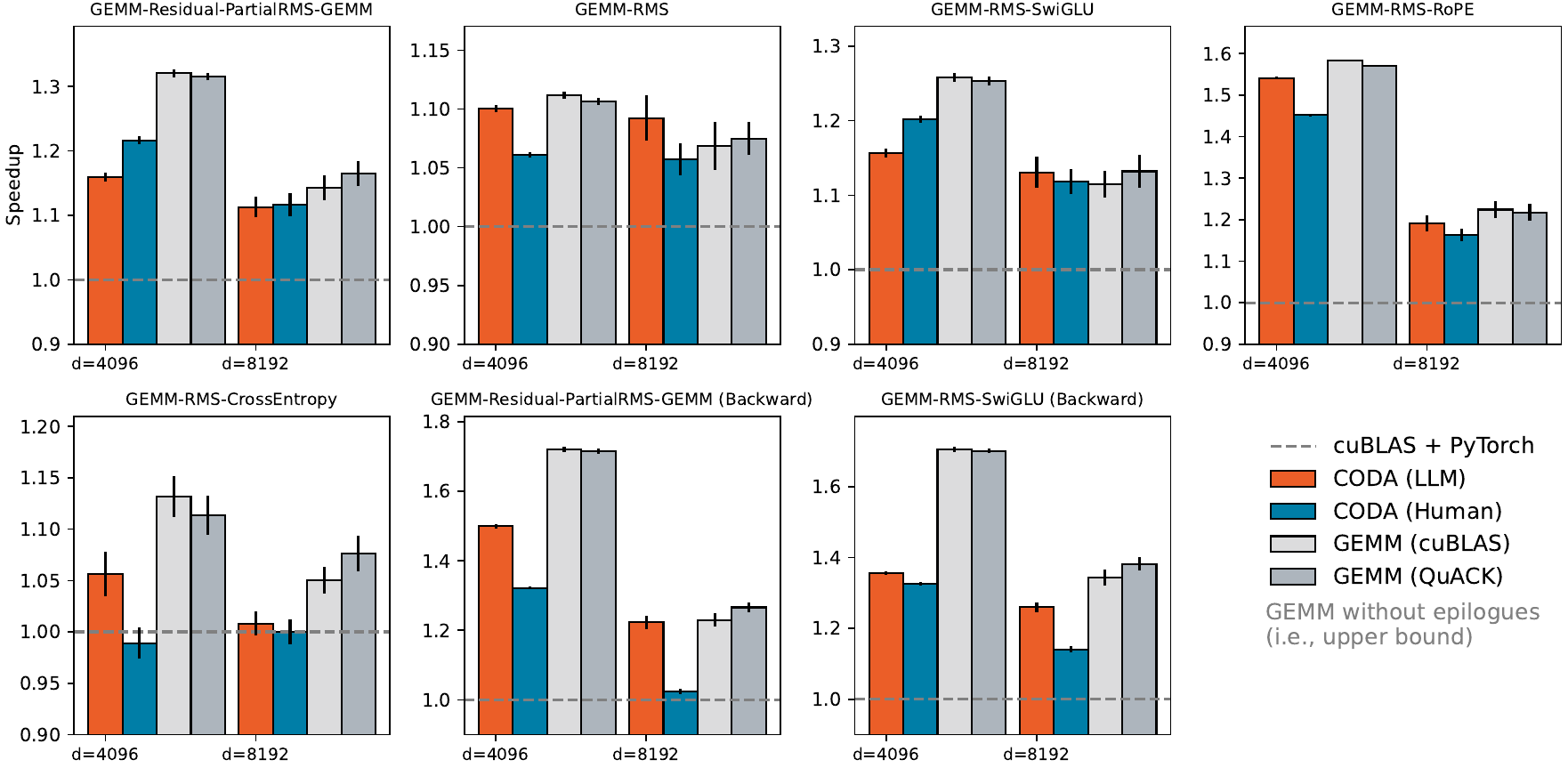}
    \vspace{-5pt}
    \caption{
    Kernel-level speedups on reparameterized Transformer kernels relative to cuBLAS with \texttt{torch.compile}. Raw GEMM baselines using PyTorch/cuBLAS and QuACK are included as reference ceilings, since they execute only the matrix multiplication and no epilogue work.
    }
    \label{fig:kernel-benchmarks}
    \vspace{-19pt}
\end{figure}

\section{Experiments}

\label{sec:setup}
After the reparameterizations in \Cref{sec:reparameterization}, we obtain a compact benchmark suite of GEMM-plus-epilogue kernels spanning the Transformer++ forward and backward pass (\Cref{sec:list-of-kernels}). The suite covers nearly all computation outside attention, embeddings, auxiliary reductions, and lightweight glue operations.
We evaluate two implementations. \texttt{CODA~(LLM)} uses Claude Code to generate most kernels from a written specification, curated examples, and a running log of implementation tips, with lightweight human supervision. \texttt{CODA~(human)} is written by human programmers using the same high-level reparameterizations, but without access to the exact \texttt{CODA} primitive set.

We compare against cuBLAS with \texttt{torch.compile}, as well as optimized LLM kernel libraries including Liger Kernel~\cite{hsu2024liger} and FlashInfer~\cite{ye2025flashinfer}. Because our reparameterized kernels do not always have one-to-one counterparts in existing libraries, we compose the closest available optimized primitives and fall back to PyTorch operators as needed. We apply \texttt{torch.compile} to each method when compatible. Additional setup details are given in \Cref{sec:setup-details}.

\paragraph{Kernel Benchmarks.}
\label{sec:kernel-benchmarks}

We first evaluate \texttt{CODA} at the individual-kernel level. Unless otherwise noted, we use square GEMM shapes with $M{=}N{=}K \in \{4096,8192\}$. For cross-entropy kernels, we set the vocabulary dimension to $32768$. For RoPE kernels, we use $N_{\mathrm{rope}}{=}3N$ to account for QKV-style projections and use precomputed $\cos$ and $\sin$ tables.\footnote{For \texttt{CODA}, we additionally pre-broadcast and extend these tables across batch, head, and QKV dimensions to avoid in-kernel branching, at the cost of additional input traffic.} For kernels that emit partial reductions, we benchmark only the fused GEMM kernel with reduction tile size $128$; auxiliary reductions are included in the block-level benchmarks below. We benchmark functions using Triton's \texttt{do\_bench} and show means and standard deviations across $30$ runs. 

We evaluate two groups of kernels from \Cref{sec:list-of-kernels}. The first group consists of standard Transformer-style kernels, such as GEMM with RoPE, SwiGLU, or cross-entropy epilogues. These kernels have close counterparts in existing libraries, so we compare against cuBLAS with \texttt{torch.compile}, Liger Kernels, and FlashInfer when applicable. Results are shown in \Cref{fig:gemm-epilogue}.

The second group consists of reparameterized Transformer forward and backward kernels, which generally do not have one-to-one equivalents in existing libraries. For these kernels, we primarily compare against cuBLAS with \texttt{torch.compile}, and additionally report raw GEMM from PyTorch/cuBLAS and QuACK.\footnote{\url{https://github.com/Dao-AILab/quack}} These raw GEMM omit epilogue work and therefore serve as reference ceilings for the attainable throughput. Results are shown in \Cref{fig:kernel-benchmarks}.

\begin{figure}[t]
    \centering
    \vspace{-15pt}
    \includegraphics[width=0.99\textwidth]{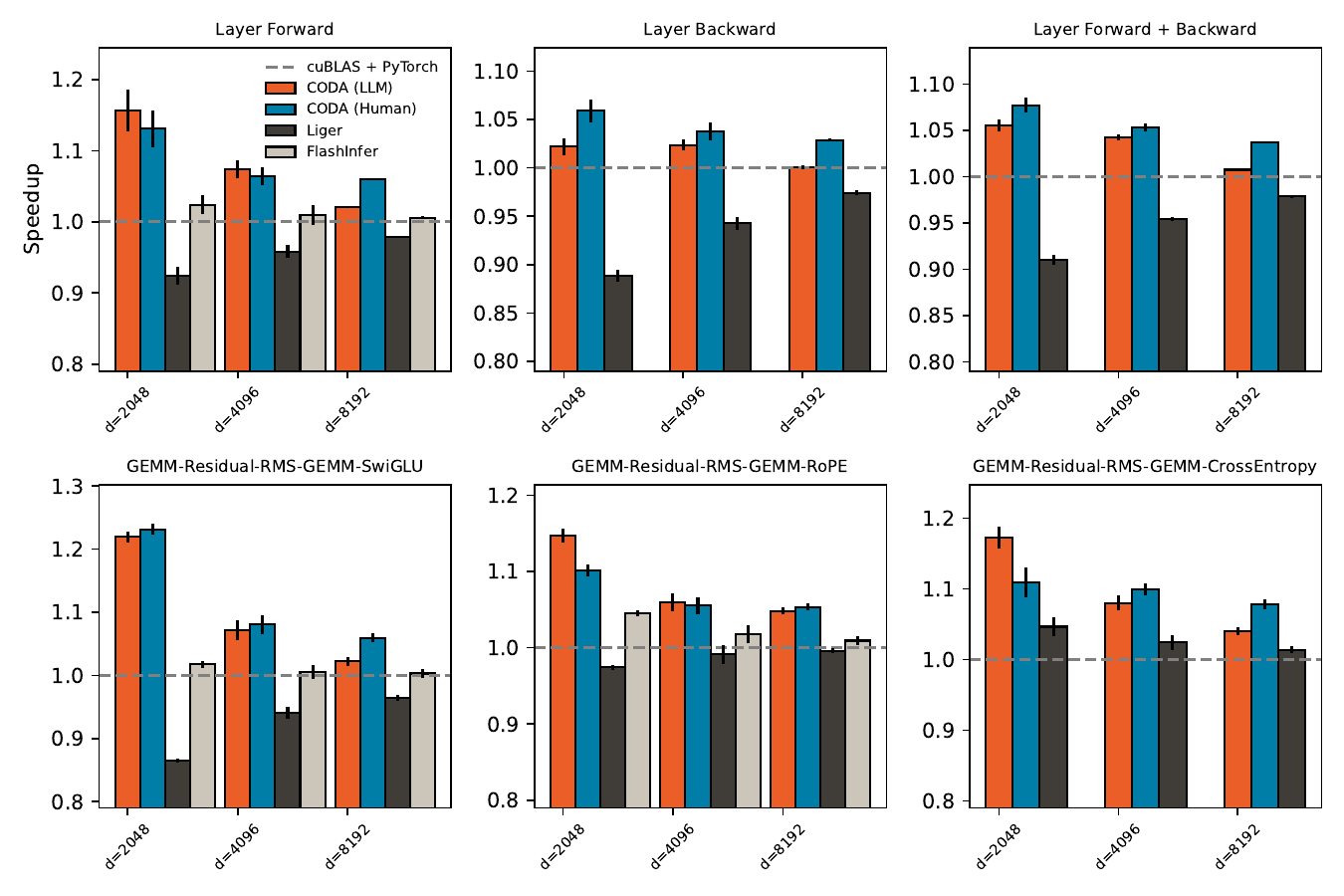}
    \vspace{-3pt}
    \caption{
    Block-level speedups for reparameterized Transformer kernel sequences, including auxiliary reductions and lightweight glue operations. Here, a layer denotes two consecutive GEMM-Residual-RMSNorm-GEMM blocks with the SwiGLU and RoPE activations, respectively.
    }
    \label{fig:block-experiments}
\end{figure}

\vspace{-2pt}
\paragraph{Block Benchmarks.}
\label{sec:block-benchmarks}

We next benchmark kernel sequences corresponding to Transformer sublayers and full layers, which we call \emph{blocks}. We use hidden sizes in $\{2048,4096,8192\}$, roughly matching 1B, 7B, and 70B model scales, with FFN expansion rate $8/3$ rounded to multiples of $256$ and vocabulary size $32768$. Unlike isolated kernel benchmarks, these measurements include auxiliary reductions and lightweight glue operations.

For the forward pass, we compare each reparameterized sequence against the closest available sequence of optimized operators. For the backward pass, the reparameterization changes the dependency structure: each sublayer emits partial statistics needed by the preceding RMSNorm backward. Individual backward sublayers therefore do not have direct PyTorch counterparts, so we report backward results at the full-layer level. In \texttt{CODA}, a layer consists of two consecutive GEMM-Residual-RMSNorm-GEMM blocks covering the SwiGLU and RoPE paths. Results are shown in \Cref{fig:block-experiments}.

\vspace{-2pt}
\section{Conclusion and Limitations}
\texttt{CODA} reparameterizes much of Transformer computation as GEMM epilogues, reducing memory-bound overhead while preserving GEMM efficiency. Its constrained abstraction supports high-performance kernels authored by both humans and LLMs.

\vspace{-2pt}
\paragraph{Limitations.}
Our reparameterizations target a common Transformer architecture; extending them to broader model families is future work. \texttt{CODA} currently focuses on single-GPU kernels and does not yet address distributed execution. Finally, while reparameterization improves efficiency, it can obscure module boundaries and algorithmic semantics, making integration with framework-level abstractions more challenging.

\section*{Acknowledgment}
We thank Beshr Islam Bouli, Kaiming Cheng, Xinle Cheng, Ryan Chin, Tarushii Goel, Wentao Guo, Lucas Torroba Hennigen, Alicia Li, Mayank Mishra, Jyothish Pari, Caiming Xiong, Nicholas Yap, Tianyuan Zhang, and Adam Zweiger for helpful discussion.
We gratefully acknowledge the support of the Schmidt Sciences AI2050 fellowship, the Google ML and Systems Junior Faculty Awards, the Google Research Scholar program, and the National Science Foundation (Award \#2441872).

\bibliographystyle{abbrvnat}
\bibliography{citations}

\appendix
\section{Backward Pass}

\subsection{Tile-wise Epilogue}
\label{sec:tile-wise-epilogue}

Partition the GEMM output $\vh$ into tiles $\vh_{[i,j]}$. A tile-wise epilogue applies an independent transformation to each tile:
\begin{align*}
\vh &= \vx \mW_0, \qquad
\vh^\prime =
\begin{bmatrix}
\vf[0,0]\left(\vh[0,0]\right) & \cdots & \vf[0,N]\left(\vh[0,N]\right) \\
\vdots & \ddots & \vdots \\
\vf[M,0]\left(\vh[M,0]\right) & \cdots & \vf[M,N]\left(\vh[M,N]\right)
\end{bmatrix},
\qquad
\vy = \vh^\prime \mW_1 .
\end{align*}

The backward pass has the same block structure.
\begin{equation*}
\footnotesize    
\begin{aligned}
\nabla_{\vh^\prime}\mathcal{L}
&= \nabla_{\vy}\mathcal{L}\mW_1^\top, \qquad
\nabla_{\vh}\mathcal{L} =
\begin{bmatrix}
\vg[0,0]\left(\nabla_{\vh^\prime}\mathcal{L}[0,0]\right) & \cdots & \vg[0,N]\left(\nabla_{\vh^\prime}\mathcal{L}[0,N]\right) \\
\vdots & \ddots & \vdots \\
\vg[M,0]\left(\nabla_{\vh^\prime}\mathcal{L}[M,0]\right) & \cdots & \vg[M,N]\left(\nabla_{\vh^\prime}\mathcal{L}[M,N]\right)
\end{bmatrix},
\qquad
\nabla_{\vx}\mathcal{L}
= \nabla_{\vh}\mathcal{L}\mW_0^\top .
\end{aligned}
\end{equation*}
Here each $\vg[i,j]$ is the local backward rule for the corresponding tile:
\begin{align*}
\vg[i,j](\Delta)
=
\operatorname{unvec}\!\left(
\mJ_{[i,j]}^\top \operatorname{vec}(\Delta)
\right),
\qquad
\mJ_{[i,j]}
=
\frac{\partial \operatorname{vec}\left(f[i,j]\left(\vh[i,j]\right)\right)}
     {\partial \operatorname{vec}\left(\vh[i,j]\right)}.
\end{align*}

Thus, each gradient tile depends only on the corresponding forward tile and upstream gradient tile. No cross-tile communication is introduced, so the backward operation remains tile-local and can be implemented as a GEMM epilogue. As shown in \Cref{fig:gemm-epilogue-fwd-bwd}, the only structural change is the direction of fusion: forward epilogues are fused into the GEMM that produces their input, while backward epilogues are fused into the GEMM that produces the gradient with respect to their output.

\subsection{GEMM-RMSNorm-GEMM Backward Pass}
\label{sec:backward-rmsnorm-epilogue}

We now describe the backward pass for the GEMM--epilogue--RMSNorm--GEMM pattern. RMSNorm is the first case where the backward pass is not purely tile-local. The reason is simple: RMSNorm contains a row-wise normalization factor, so its backward pass needs a row-wise statistic. In addition, the RMSNorm weight $\boldsymbol{\gamma}$ is shared across rows, so its gradient requires a reduction across the row dimension. The goal of this section is to show that these are the only non-local pieces. Everything else can still be fused into GEMM epilogues, with the non-local pieces handled by lightweight reductions over tile partials.

Consider the forward computation
\begin{align*}
\vh_0 &= \vx \mW_0, \qquad
\vh_1 = f(\vh_0), \qquad
\vh_2 = \operatorname{RMSNorm}(\vh_1, \boldsymbol{\gamma}), \qquad
\vy = \vh_2 \mW_1 .
\end{align*}
Let $\vr$ be the row-wise inverse RMS factor. We write
$\overline{\vr} = \vr \mathbf{1}^{\top}$ and
$\overline{\boldsymbol{\gamma}} = \mathbf{1}\boldsymbol{\gamma}^{\top}$
for the broadcasts of $\vr$ and $\boldsymbol{\gamma}$ to the shape of $\vh_1$. Then
\begin{align*}
\vh_2
=
\overline{\vr} \odot \vh_1 \odot \overline{\boldsymbol{\gamma}} .
\end{align*}

Given the upstream gradient $\nabla_{\vy}\mathcal{L}$, the first backward operation is the GEMM
\begin{align*}
\nabla_{\vh_2}\mathcal{L}
=
\nabla_{\vy}\mathcal{L}\mW_1^\top .
\end{align*}
The RMSNorm backward can be written as
\begin{align*}
\nabla_{\vh_1}\mathcal{L}
&=
\overline{\vr} \odot
\left(
\nabla_{\vh_2}\mathcal{L} \odot \overline{\boldsymbol{\gamma}}
-
\overline{\vr} \odot \vh_1 \odot \overline{\vs}
\right), \\
\nabla_{\boldsymbol{\gamma}}\mathcal{L}
&=
\operatorname{sum}_{\mathrm{rows}}
\left(
\nabla_{\vh_2}\mathcal{L}
\odot \vh_1
\odot \overline{\vr}
\right),
\end{align*}
where $\overline{\vs} = \vs \mathbf{1}^{\top}$ broadcasts one scalar per row. The row-wise statistic $\vs$ is
\begin{align*}
\vs
&=
\frac{1}{d}
\odot \vr \odot
\operatorname{sum}_{\mathrm{cols}}
\left(
\nabla_{\vh_2}\mathcal{L} \odot \overline{\boldsymbol{\gamma}} \odot \vh_1
\right), \\
&=
\frac{1}{d}
\operatorname{sum}_{\mathrm{cols}}
\left(
\nabla_{\vh_2}\mathcal{L} \odot \overline{\vr} \odot \overline{\boldsymbol{\gamma}} \odot \vh_1
\right), \\
&=
\frac{1}{d}
\operatorname{sum}_{\mathrm{cols}}
\left(
\nabla_{\vh_2}\mathcal{L} \odot \vh_2
\right),
\end{align*}
where $d$ is the hidden dimension. This expression identifies the two non-local operations in RMSNorm backward. The statistic $\vs$ is a reduction across columns, producing one scalar per row. The weight gradient $\nabla_{\boldsymbol{\gamma}}\mathcal{L}$ is a reduction across rows, producing one scalar per hidden feature.

A standalone RMSNorm backward kernel would compute these reductions by reading activation-sized tensors. The key observation is that the row-wise statistic $\vs$ can be moved to a different boundary. Using $\nabla_{\vh_2}\mathcal{L} = \nabla_{\vy}\mathcal{L} \, \mW_1^\top$ and
$\vy = \vh_2 \mW_1$, we have
\begin{align*}
\vs
&=
\frac{1}{d}
\operatorname{sum}_{\mathrm{cols}}
\left(
\nabla_{\vh_2}\mathcal{L} \odot \vh_2
\right), \\
&=
\frac{1}{d}
\operatorname{sum}_{\mathrm{cols}}
\left(
(\nabla_{\vy}\mathcal{L}\mW_1^\top) \odot \vh_2
\right) \\
&=
\frac{1}{d}
\operatorname{diag}
\left(
(\nabla_{\vy}\mathcal{L}\mW_1^\top) \; \vh_2^\top
\right) \\
&=
\frac{1}{d}
\operatorname{diag}
\left(
\nabla_{\vy}\mathcal{L} \; (\vh_2 \mW_1)^\top
\right) \\
&=
\frac{1}{d}
\operatorname{sum}_{\mathrm{cols}}
\left(
\nabla_{\vy}\mathcal{L} \odot (\vh_2 \mW_1)
\right) \\
&=
\frac{1}{d}
\operatorname{sum}_{\mathrm{output}}
\left(
\nabla_{\vy}\mathcal{L} \odot \vy
\right).
\end{align*}
Intuitively, the RMSNorm backward needs the inner product between an activation and its gradient along each row. The identity above says that this inner product can be computed either before or after the following GEMM. This lets us compute the statistic at a boundary where $\vy$ and $\nabla_{\vy}\mathcal{L}$ are already available.

This is useful because Transformer layers contain consecutive GEMM--epilogue--RMSNorm--GEMM patterns. During the backward pass of one pattern, the GEMM that produces $\nabla_{\vx}\mathcal{L}$ already has access to both $\vx$ and $\nabla_{\vx}\mathcal{L}$. Since this $\vx$ is the output of the preceding pattern, the epilogue of the current pattern can accumulate the RMSNorm statistic needed by the preceding pattern:
\begin{align*}
\widehat{\vs}_{\mathrm{prev}}
=
\operatorname{reduceTile}_{\mathrm{cols}}
\left(
\vx \odot \nabla_{\vx}\mathcal{L}
\right),
\qquad
\vs_{\mathrm{prev}}
=
\frac{1}{d}
\operatorname{reduce}
\left(
\widehat{\vs}_{\mathrm{prev}}
\right).
\end{align*}
Thus, each pattern computes the row-wise RMSNorm backward statistic required by the pattern before it. The reduction is still present, but it is now a small reduction over tile partials rather than a standalone activation-sized RMSNorm backward kernel.

The RMSNorm weight gradient is handled similarly, except that its reduction is across rows rather than columns. We accumulate tile partials in the RMSNorm backward epilogue:
\begin{align*}
\widehat{\nabla_{\boldsymbol{\gamma}}\mathcal{L}}
=
\operatorname{reduceTile}_{\mathrm{rows}}
\left(
\nabla_{\vh_2}\mathcal{L}
\odot \vh_1
\odot \overline{\vr}
\right),
\qquad
\nabla_{\boldsymbol{\gamma}}\mathcal{L}
=
\operatorname{reduce}_{\mathrm{rows}}
\left(
\widehat{\nabla_{\boldsymbol{\gamma}}\mathcal{L}}
\right).
\end{align*}

Putting these pieces together, the backward pass is organized as follows:
\begin{align*}
\text{GEMM 1:}\quad
\nabla_{\vh_2}\mathcal{L}
&=
\nabla_{\vy}\mathcal{L}\mW_1^\top, \\
\text{Epilogue 1:}\quad
\nabla_{\vh_1}\mathcal{L}
&=
\overline{\vr} \odot
\left(
\nabla_{\vh_2}\mathcal{L} \odot \overline{\boldsymbol{\gamma}}
-
\vh_1 \odot \overline{\vr} \odot \overline{\vs}
\right), \\
\nabla_{\vh_0}\mathcal{L}
&=
g(\nabla_{\vh_1}\mathcal{L}), \\
\widehat{\nabla_{\boldsymbol{\gamma}}\mathcal{L}}
&=
\operatorname{reduceTile}_{\mathrm{rows}}
\left(
\nabla_{\vh_2}\mathcal{L}
\odot \vh_1
\odot \overline{\vr}
\right), \\ \\
\text{GEMM 2:}\quad
\nabla_{\vx}\mathcal{L}
&=
\nabla_{\vh_0}\mathcal{L}\mW_0^\top, \\
\text{Epilogue 2:}\quad
\widehat{\vs}_{\mathrm{prev}}
&=
\operatorname{reduceTile}_{\mathrm{cols}}
\left(
\vx \odot \nabla_{\vx}\mathcal{L}
\right), \\ \\
\text{Auxiliary reductions:}\quad
\vs_{\mathrm{prev}}
&=
\frac{1}{d}
\operatorname{reduce}_{\mathrm{cols}}
\left(
\widehat{\vs}_{\mathrm{prev}}
\right), \\
\nabla_{\boldsymbol{\gamma}}\mathcal{L}
&=
\operatorname{reduce}_{\mathrm{rows}}
\left(
\widehat{\nabla_{\boldsymbol{\gamma}}\mathcal{L}}
\right).
\end{align*}
Here $g$ denotes the tile-local backward rule for the epilogue $f$. The statistic $\vs$ used in \text{Epilogue 1} is assumed to have already been accumulated by the following pattern in the backward order.

Finally, the output of a pattern often passes through another epilogue before the next pattern begins:
\begin{align*}
\vh_0 &= \vx \mW_0, \qquad
\vh_1 = f_0(\vh_0), \qquad
\vh_2 = \operatorname{RMSNorm}(\vh_1, \boldsymbol{\gamma}), \qquad
\vh_3 = \vh_2 \mW_1
\qquad \vy = f_1(\vh_3).
\end{align*}
In this case, the statistic should be accumulated using the pre-epilogue tensor $\vh_3$ and its gradient. The backward rule for $f_1$ is tile-local, so it can be fused before accumulating the statistic:
\begin{align*}
\text{GEMM 2:}\quad
\nabla_{\vx}\mathcal{L}
&=
\nabla_{\vh_0}\mathcal{L}\mW_0^\top, \\
\text{Epilogue 2:}\quad
\nabla_{\overleftarrow{\vh_3}}\mathcal{L}
&=
\overleftarrow{f_1}(\nabla_{\vx}\mathcal{L}), \\
\widehat{\vs}_{\mathrm{prev}}
&=
\operatorname{reduceTile}_{\mathrm{cols}}
\left(
\overleftarrow{\vh_3} \odot \nabla_{\overleftarrow{\vh_3}}\mathcal{L}
\right)
\end{align*}
Here $\overleftarrow{\vh_3}$ denotes the pre-epilogue GEMM output from the preceding pattern, and $\overleftarrow{f_1}$ denotes the local backward rule for its following epilogue. This case preserves the same structure: apply the local backward epilogue first, then accumulate the row-wise statistic from the pre-epilogue activation and its gradient.

Overall, this organization removes the activation-sized RMSNorm backward kernel. The remaining non-local work consists only of reductions over tile partials: a column reduction for the row-wise RMSNorm statistic, and a row reduction for the RMSNorm weight gradient. The GEMMs and local backward updates are fused into GEMM epilogues, preserving the same GEMM-plus-epilogue structure used in the forward pass.

\section{CODA}
\subsection{Epilogue Template}
\label{label:evt-template}
\begin{lstlisting}[caption={Epilogue Kernel Abstraction.},label={lst:engine}, language=Python,basicstyle=\ttfamily\scriptsize]
# --- Mainloop ---
# ...
# --- Pre-loop: visitor tree entry points ---
epilogue.consumer_begin(...)
epilogue.producer_begin(...)

# --- Main loop: epilogue tiles ---
for epi_idx in range(num_epi_tiles):
    # Per-tile visitor entry
    epilogue.consumer_begin_loop(...)
    # Producer-side async load (e.g., TMA)
    epilogue.producer_tma_load(...)
    # Core epilogue computation
    rD = load_accumulator_fragment(...)
    epilogue.consumer_visit(rD, ...)
    # Registers -> shared memory
    store_regs_to_smem(...)
    # Pre-store visitor hook
    epilogue.consumer_smem_store(...)
    tma_store_from_smem_to_gmem(...)
    epilogue.consumer_tma_store(...)
    # Per-tile visitor exit
    epilogue.consumer_end_loop(gmem_coord)

# --- Post-loop: visitor tree finalization ---
epilogue.consumer_end(...)
\end{lstlisting}

\subsection{Epilogue Example}
\label{label:evt-example}
\begin{lstlisting}[caption={Kernel Example.},label={lst:rapier-example}, language=Python,basicstyle=\ttfamily\scriptsize]
def _create_mean_sq_reduction_op(element_type, inv_block_size):
    """Create a reduction op that accumulates mean of squares: acc + val^2 * inv_block_size.

    The combine_fn squares each new element and scales by 1/block_size before adding
    to the accumulator. The warp-level reduction uses standard addition since partial
    sums are already accumulated and scaled.
    """
    init_value = element_type(0.)
    inv_bs = element_type(inv_block_size)

    _sq_combine = lambda x, y: x + y * y * inv_bs
    _add_wrp = lambda tree_x, tree_y: pytree.tree_map(operator.add, tree_x, tree_y)

    return BlockReductionOp(
        combine_fn=lambda tree_x, tree_y: pytree.tree_map(_sq_combine, tree_x, tree_y),
        reduce_ssa=None,
        reduce_wrp=lambda xs: pytree.tree_map(
            lambda x: cute.arch.warp_reduction(
                x,
                op=_add_wrp,
                threads_in_group=HOPPER_WARP_REDUCTION_WIDTH,
            ),
            xs,
        ),
        init_value=init_value,
    )


class EVTRowVecMulPostAct(EpilogueVisitorTree):
    """
    Loads a per-N row vector W (cp.async to smem, then s2r), multiplies the
    accumulator by W into a separate register tile, and stores that scaled
    tile to a side output mPostAct via TMA. tRS_rD itself is left unchanged
    so the main D output (the unscaled GEMM result) is unaffected.

    This mirrors the rowvec=norm_weight side-output path of trainstation's
    `gemm_partial_rms_fwd`, kept local to this kernel rather than as a
    general-purpose rapier EVT.

    Inputs:
        - GEMM output (in registers): [M x N], unchanged by this op
        - mRowVec: [L, N]  - RMSNorm weight, broadcast along M

    Outputs:
        - mPostAct: [M x N] = D * W  (side output, written via TMA)
    """

    @struct_utils.mlir_namedtuple
    class EpilogueArguments(NamedTuple):
        mPostAct: cute.Tensor | None
        mRowVec: cute.Tensor | None

    @struct_utils.register_pytree_dataclass
    @dataclass
    class EpilogueParams(EpilogueVisitorTree.EpilogueParams):
        mPostAct: cute.Tensor | None
        mRowVec: cute.Tensor | None
        epi_tma_atom: cute.CopyAtom
        epi_gmem_layout: cutlass.utils.LayoutEnum
        epi_smem_layout_staged: cute.Layout

    @struct_utils.register_pytree_dataclass
    @dataclass
    class EpilogueTensorsSMem(EpilogueVisitorTree.EpilogueTensorsSMem):
        sPostAct: cute.Tensor | None
        sRowVec: cute.Tensor | None

    @struct_utils.register_pytree_dataclass
    @dataclass
    class EpilogueTensors(EpilogueVisitorTree.EpilogueTensors):
        tDsPostAct: cute.Tensor
        tDgPostAct: cute.Tensor
        tRS_sPostAct: cute.Tensor
        epi_tma_atom: cute.CopyAtom
        tiled_copy_postact_r2s: cute.TiledCopy
        tDsRowVec: cute.Tensor | None

    @struct_utils.register_pytree_dataclass
    @dataclass
    class EpilogueTensorsLoop(EpilogueVisitorTree.EpilogueTensorsLoop):
        tDsPostAct: cute.Tensor
        tDgPostAct: cute.Tensor
        tRS_rPostAct: cute.Tensor | None
        tRS_sPostAct: cute.Tensor
        epi_tma_atom: cute.CopyAtom
        tiled_copy_postact_r2s: cute.TiledCopy
        tDrRowVec_epi: cute.Tensor | None

    @struct_utils.register_pytree_dataclass
    @dataclass
    class EpiloguePipelines(EpilogueVisitorTree.EpiloguePipelines):
        pass

    def __init__(
        self,
        acc_dtype: type[cute.Numeric],
        post_act_dtype: type[cute.Numeric],
        tile_shape_mnk: tuple[int, int, int],
        buffer_align_bytes: int,
    ) -> None:
        super().__init__()
        self.arch = 90
        self.acc_dtype = acc_dtype
        self.post_act_dtype = post_act_dtype
        self.container_dtype = post_act_dtype
        self.tile_shape_mnk = tile_shape_mnk
        self.buffer_align_bytes = buffer_align_bytes

    @cute.jit
    def to_underlying_arguments(
        self,
        epi_tile: cute.Tile,
        epi_stage: int,
        epi_load_stage: int,
        epi_args: EpilogueArguments,
    ) -> EpilogueParams:

        if cutlass.const_expr(epi_args.mPostAct is not None):
            mPostAct = misc_utils.static_assert_is_Tensor(epi_args.mPostAct)
            misc_utils.static_assert(get_dtype(mPostAct) is self.container_dtype)
            (
                epi_gmem_layout,
                epi_smem_layout_staged,
                epi_tma_atom,
                epi_tma_tensor,
            ) = epilogue_utils.prepare_tma(
                tma_op="s2g",
                epi_tile=epi_tile,
                epi_stage=epi_stage,
                epi_tensor=mPostAct,
            )

        if cutlass.const_expr(epi_args.mRowVec is not None):
            misc_utils.static_assert(epi_args.mPostAct is not None)
            mRowVec = misc_utils.static_assert_is_Tensor(epi_args.mRowVec)
            mRowVec = layout_utils.assumed_align_stride(
                mRowVec,
                assumed_align=4,
            )
        else:
            mRowVec = None

        return self.EpilogueParams(
            mPostAct=epi_tma_tensor,
            mRowVec=mRowVec,
            epi_tma_atom=epi_tma_atom,
            epi_gmem_layout=epi_gmem_layout,
            epi_smem_layout_staged=epi_smem_layout_staged,
        )

    @cute.jit
    def prefetch_tma_descriptors(
        self,
        epi_params: EpilogueParams,
    ) -> None:
        cute.nvgpu.cpasync.prefetch_descriptor(epi_params.epi_tma_atom)

    @cute.jit
    def consumer_begin(
        self,
        tiled_copy_r2s: cute.TiledCopy,
        tile_coord_mnkl: cute.Coord,
        tidx: cute.Int32,
        tiled_mma: cute.TiledMma,
        tRS_rD_layout: cute.Layout,
        epi_tile: cute.Tile,
        epi_num_threads: int,
        epi_num_matrices: int,
        epi_barrier: cutlass.pipeline.NamedBarrier,
        epi_params: EpilogueParams,
        epi_tensors_smem: EpilogueTensorsSMem,
    ) -> EpilogueTensors:

        tile_M = self.tile_shape_mnk[0]
        tile_N = self.tile_shape_mnk[1]
        m_idx, n_idx, _, batch_idx = tile_coord_mnkl
        thr_copy_r2s = tiled_copy_r2s.get_slice(tidx)

        # Side output (PostAct) TMA setup
        mPostAct = misc_utils.static_assert_is_Tensor(epi_params.mPostAct)
        sPostAct = misc_utils.static_assert_is_Tensor(epi_tensors_smem.sPostAct)
        tiled_copy_postact_r2s, _, tRS_sPostAct = epilogue_utils.prepare_copy_r2s_sm90(
            tiled_copy_r2s=tiled_copy_r2s,
            tidx=tidx,
            dst=sPostAct,
            epi_layout=epi_params.epi_gmem_layout,
            epi_dtype=self.container_dtype,
            acc_dtype=self.acc_dtype,
        )
        gPostAct = mPostAct[None, None, batch_idx]
        gPostAct = cute.local_tile(gPostAct, (tile_M, tile_N), (m_idx, n_idx))
        gPostAct = cute.zipped_divide(gPostAct, epi_tile)

        tDsPostAct, tDgPostAct = cute.nvgpu.cpasync.tma_partition(
            atom=epi_params.epi_tma_atom,
            cta_coord=0,
            cta_layout=cute.make_layout(1),
            smem_tensor=cute.group_modes(sPostAct, 0, cute.rank(sPostAct) - 1),
            gmem_tensor=cute.group_modes(gPostAct, 0, cute.rank(gPostAct) - 1),
        )

        # RowVec cp.async load (per-N broadcast across M)
        if cutlass.const_expr(epi_params.mRowVec is not None):
            mRowVec = misc_utils.static_assert_is_Tensor(epi_params.mRowVec)
            sRowVec = misc_utils.static_assert_is_Tensor(epi_tensors_smem.sRowVec)
            mRowVec = mRowVec[batch_idx, None]
            gRowVec = cute.local_tile(mRowVec, (tile_N,), (n_idx,))
            cRowVec = cute.make_identity_tensor(tile_N)
            limit_n = min(mRowVec.shape[0] - n_idx * tile_N, tile_N)
            memory_utils.g2s_copy_1d(
                src=gRowVec,
                dst=sRowVec,
                crd=cRowVec,
                shape=(limit_n,),
                num_threads=epi_num_threads,
                thread_index=tidx,
            )
            sRowVec_view_layout = cute.make_layout(
                shape=(tile_M, tile_N),
                stride=(0, 1),
            )
            sRowVec_view = cute.make_tensor(
                iterator=sRowVec.iterator,
                layout=sRowVec_view_layout,
            )
            tDsRowVec = thr_copy_r2s.partition_S(
                cute.flat_divide(sRowVec_view, epi_tile)
            )
            cute.arch.cp_async_commit_group()
            cute.arch.cp_async_wait_group(0)
            epi_barrier.arrive_and_wait()
        else:
            tDsRowVec = None

        return self.EpilogueTensors(
            tDsPostAct=tDsPostAct,
            tDgPostAct=tDgPostAct,
            tRS_sPostAct=tRS_sPostAct,
            epi_tma_atom=epi_params.epi_tma_atom,
            tiled_copy_postact_r2s=tiled_copy_postact_r2s,
            tDsRowVec=tDsRowVec,
        )

    @cute.jit
    def consumer_end(
        self,
        tiled_copy_r2s: cute.TiledCopy,
        tile_coord_mnkl: cute.Coord,
        tidx: cute.Int32,
        shape_mnk: cute.Shape,
        epi_tile: cute.Tile,
        epi_num_threads: int,
        epi_barrier: cutlass.pipeline.NamedBarrier,
        epi_params: EpilogueParams,
        epi_tensors: EpilogueTensors,
        epi_tensors_smem: EpilogueTensorsSMem,
    ) -> None:
        pass

    @cute.jit
    def consumer_begin_loop(
        self,
        epi_coord: cute.Coord,
        epi_params: EpilogueParams,
        epi_tensors: EpilogueTensors,
        epi_pipelines: EpiloguePipelines,
    ) -> tuple[EpilogueTensorsLoop, EpiloguePipelines]:

        if cutlass.const_expr(epi_tensors.tDsRowVec is not None):
            tDsRowVec = misc_utils.static_assert_is_Tensor(epi_tensors.tDsRowVec)
            tDsRowVec_cur = cute.group_modes(tDsRowVec, 3, cute.rank(tDsRowVec))
            tDsRowVec_cur = tDsRowVec_cur[None, None, None, epi_coord]
            tDrRowVec_cvt = memory_utils.s2r_copy_1d(tDsRowVec_cur, dtype=self.acc_dtype)
        else:
            tDrRowVec_cvt = None

        return (
            self.EpilogueTensorsLoop(
                tDsPostAct=epi_tensors.tDsPostAct,
                tDgPostAct=epi_tensors.tDgPostAct,
                tRS_rPostAct=None,
                tRS_sPostAct=epi_tensors.tRS_sPostAct,
                epi_tma_atom=epi_tensors.epi_tma_atom,
                tiled_copy_postact_r2s=epi_tensors.tiled_copy_postact_r2s,
                tDrRowVec_epi=tDrRowVec_cvt,
            ),
            self.EpiloguePipelines(),
        )

    @cute.jit
    def consumer_visit(
        self,
        tRS_rD: cute.Tensor,
        shape_mnk: cute.Shape,
        epi_params: EpilogueParams,
        epi_tensors_loop: EpilogueTensorsLoop,
    ) -> EpilogueTensorsLoop:

        tRS_rPostAct = creation_utils.allocate_tensor_like(
            tensor=tRS_rD,
            memspace="rmem",
            smem_allocator=None,
            dtype=self.acc_dtype,
        )
        if cutlass.const_expr(self.arch < 100):
            if cutlass.const_expr(epi_tensors_loop.tDrRowVec_epi is not None):
                tDrRowVec_epi = misc_utils.static_assert_is_Tensor(epi_tensors_loop.tDrRowVec_epi)
                for i in cutlass.range_constexpr(cute.size(tRS_rPostAct)):
                    tRS_rPostAct[i] = tRS_rD[i] * tDrRowVec_epi[i]
            else:
                for i in cutlass.range_constexpr(cute.size(tRS_rPostAct)):
                    tRS_rPostAct[i] = tRS_rD[i]
        else:
            raise NotImplementedError

        tRS_rPostAct = dtype_utils.convert(
            tRS_rPostAct,
            dtype=self.post_act_dtype,
        )

        return self.EpilogueTensorsLoop(
            tDsPostAct=epi_tensors_loop.tDsPostAct,
            tDgPostAct=epi_tensors_loop.tDgPostAct,
            tRS_rPostAct=tRS_rPostAct,
            tRS_sPostAct=epi_tensors_loop.tRS_sPostAct,
            epi_tma_atom=epi_tensors_loop.epi_tma_atom,
            tiled_copy_postact_r2s=epi_tensors_loop.tiled_copy_postact_r2s,
            tDrRowVec_epi=epi_tensors_loop.tDrRowVec_epi,
        )

    @cute.jit
    def consumer_smem_store(
        self,
        epi_coord: cute.Coord,
        epi_buffer: cute.Int32,
        epi_params: EpilogueParams,
        epi_tensors_loop: EpilogueTensorsLoop,
    ) -> None:
        tiled_copy = epi_tensors_loop.tiled_copy_postact_r2s
        tRS_rPostAct = misc_utils.static_assert_is_Tensor(epi_tensors_loop.tRS_rPostAct)
        tRS_sPostAct = misc_utils.static_assert_is_Tensor(epi_tensors_loop.tRS_sPostAct)
        src = tiled_copy.retile(tRS_rPostAct)
        dst = tRS_sPostAct[None, None, None, epi_buffer]
        cute.copy(atom=tiled_copy, src=src, dst=dst)

    @cute.jit
    def consumer_tma_store(
        self,
        epi_coord: cute.Coord,
        epi_buffer: cute.Int32,
        epi_params: EpilogueParams,
        epi_tensors_loop: EpilogueTensorsLoop,
    ) -> None:
        atom = epi_tensors_loop.epi_tma_atom
        tDsPostAct = misc_utils.static_assert_is_Tensor(epi_tensors_loop.tDsPostAct)
        tDgPostAct = misc_utils.static_assert_is_Tensor(epi_tensors_loop.tDgPostAct)
        src = tDsPostAct[None, epi_buffer]
        dst = tDgPostAct[None, epi_coord]
        cute.copy(atom=atom, src=src, dst=dst)

    @cute.jit
    def get_smem_struct(
        self,
        epi_load_stage: int,
        epi_num_threads: int,
        epi_params: EpilogueParams,
    ) -> type[EpilogueSharedStorage]:

        if cutlass.const_expr(epi_params.mPostAct is not None):
            post_act_smem_size = cute.cosize(epi_params.epi_smem_layout_staged)
        else:
            post_act_smem_size = 0

        if cutlass.const_expr(epi_params.mRowVec is not None):
            mRowVec = misc_utils.static_assert_is_Tensor(epi_params.mRowVec)
            row_vec_dtype = get_dtype(mRowVec)
            row_vec_smem_size = epilogue_utils.get_smem_size_vector(
                mTensor=mRowVec,
                epi_tile=self.tile_shape_mnk[1],
                epi_num_threads=epi_num_threads,
            )
        else:
            row_vec_dtype = cute.Float32
            row_vec_smem_size = 0

        @cute.struct
        class SharedStorage(EpilogueSharedStorage):
            sPostAct: cute.struct.Align[cute.struct.MemRange[self.container_dtype, post_act_smem_size], self.buffer_align_bytes]
            sRowVec: cute.struct.Align[cute.struct.MemRange[row_vec_dtype, row_vec_smem_size], 16]

        return SharedStorage

    @cute.jit
    def get_smem_tensors(
        self,
        storage: EpilogueSharedStorage,
        epi_num_threads: int,
        epi_params: EpilogueParams,
    ) -> EpilogueTensorsSMem:

        if cutlass.const_expr(epi_params.mPostAct is not None):
            sPostAct = storage.sPostAct.get_tensor(
                epi_params.epi_smem_layout_staged.outer,
                swizzle=epi_params.epi_smem_layout_staged.inner,
            )
        else:
            sPostAct = None

        if cutlass.const_expr(epi_params.mRowVec is not None):
            sRowVec_layout = cute.make_layout(self.tile_shape_mnk[1])
            sRowVec = storage.sRowVec.get_tensor(sRowVec_layout)
        else:
            sRowVec = None

        return self.EpilogueTensorsSMem(
            sPostAct=sPostAct,
            sRowVec=sRowVec,
        )

    @cute.jit
    def get_smem_bytes_per_stage(
        self,
        epi_tile: cute.Tile,
        epi_num_threads: int,
        epi_args: EpilogueArguments,
    ) -> tuple[int, int, int]:
        epi_smem_bytes_fixed = 0
        epi_smem_bytes_per_stage_cst = 0
        epi_smem_bytes_per_stage_pld = 0

        if cutlass.const_expr(epi_args.mPostAct is not None):
            mPostAct = misc_utils.static_assert_is_Tensor(epi_args.mPostAct)
            misc_utils.static_assert(get_dtype(mPostAct) is self.container_dtype)
            epi_smem_bytes_per_stage_cst = epi_smem_bytes_per_stage_cst + (
                epilogue_utils.get_epi_smem_bytes_per_stage_matrix(
                    mTensor=mPostAct,
                    epi_tile=epi_tile,
                )
            )

        if cutlass.const_expr(epi_args.mRowVec is not None):
            mRowVec = misc_utils.static_assert_is_Tensor(epi_args.mRowVec)
            epi_smem_bytes_fixed = epi_smem_bytes_fixed + (
                epilogue_utils.get_epi_smem_bytes_per_stage_fixed_vector(
                    mTensor=mRowVec,
                    epi_tile=self.tile_shape_mnk[1],
                    epi_num_threads=epi_num_threads,
                )
            )

        return (
            epi_smem_bytes_fixed,
            epi_smem_bytes_per_stage_cst,
            epi_smem_bytes_per_stage_pld,
        )


def prepare_epilogue(
    shape_mnkl: tuple[int, int, int, int],
    tile_shape_mn: tuple[int, int],
    C: torch.Tensor,
    S: torch.Tensor,
    W: torch.Tensor,
    O: torch.Tensor,
) -> tuple[
    Callable[..., EpilogueVisitorTree],
    EpilogueVisitorTree.EpilogueArguments,
    dict,
    tuple,
]:
    """Prepare epilogue for GEMM with residual, partial mean-of-squares, and
    fused per-N RMSNorm-weight scaling - mirrors trainstation's `gemm_partial_rms_fwd`.

    Composes three EVT visitors:
        1. EVTResidual: D = acc + C
        2. EVTColBlockReductionStore: S[m, nb] = mean(D[m, nb*bs:(nb+1)*bs]^2)
        3. EVTRowVecMulPostAct (local): O[m, n] = D[m, n] * W[n], side output via TMA

    The partial sum-of-squares is computed on the *unscaled* D, so a downstream
    rstd reduction sees the GEMM output before W is applied. tRS_rD is preserved
    so the main D output is also unscaled.

    Args:
        shape_mnkl: Problem shape (M, N, K, L) where L is batch dimension.
        tile_shape_mn: CTA tile shape (tile_M, tile_N).
        C: Residual matrix of shape (M, N).
        S: Output for partial mean-of-squares of shape (M, num_blocks) in fp32.
        W: RMSNorm weight of shape (N,), broadcast across M.
        O: Output of shape (M, N) for D * W.

    Returns:
        Tuple of (epi_cls, epi_args, epi_outs, epi_keys).
    """
    M, N, K, L = shape_mnkl

    epi_dtype = torch2cute_dtype_map[C.dtype]
    post_act_dtype = torch2cute_dtype_map[O.dtype]

    epi_cls = lambda acc_dtype, tile_shape_mnk, buffer_align_bytes: EVTList([
        EVTResidual(
            acc_dtype=acc_dtype,
            epi_dtype=epi_dtype,
            tile_shape_mnk=tile_shape_mnk,
            buffer_align_bytes=buffer_align_bytes,
        ),
        EVTColBlockReductionStore(
            reduction_op=_create_mean_sq_reduction_op(
                element_type=acc_dtype,
                inv_block_size=1.0 / tile_shape_mnk[1],
            ),
            tile_shape_mnk=tile_shape_mnk,
        ),
        EVTRowVecMulPostAct(
            acc_dtype=acc_dtype,
            post_act_dtype=post_act_dtype,
            tile_shape_mnk=tile_shape_mnk,
            buffer_align_bytes=buffer_align_bytes,
        ),
    ])

    epi_args = EVTList.EpilogueArguments([
        EVTResidual.EpilogueArguments(
            mMatrix=C,
        ),
        EVTColBlockReductionStore.EpilogueArguments(
            mColVec=S,
        ),
        EVTRowVecMulPostAct.EpilogueArguments(
            mPostAct=O,
            mRowVec=W,
        ),
    ])

    epi_keys = (
        C.dtype,
        S.dtype,
        W.dtype,
        O.dtype,
        EVTResidual,
        EVTColBlockReductionStore,
        EVTRowVecMulPostAct,
    )

    epi_outs = {}

    return epi_cls, epi_args, epi_outs, epi_keys
\end{lstlisting}

\section{Experiments}
\subsection{List of Kernels}
\label{sec:list-of-kernels}

We summarize the kernels implemented in \texttt{CODA}. Each kernel is a GEMM followed by an epilogue program.

\subsubsection{Basic Epilogue Kernels}

We first list three basic GEMM-plus-epilogue kernels. These are useful for isolating individual epilogue primitives, although they are not always used directly in the Transformer forward pass.

\textbf{Kernel 1: GEMM with RoPE.}
This kernel applies RoPE~\cite{su2024roformer} to pairs of adjacent features in the GEMM output:
\begin{align*}
\mD &= \mA\mB, \\
\mO &= \operatorname{RoPE}(\mD).
\end{align*}

\textbf{Kernel 2: GEMM with SwiGLU.}
This kernel applies a fused SwiGLU activation to an interleaved GEMM output:
\begin{align*}
\mD &= \mA\mB, \\
[\mG,\mU] &= \operatorname{interleavedSplit}(\mD), \\
\mO &= \operatorname{silu}(\mG) \odot \mU.
\end{align*}

\textbf{Kernel 3: GEMM with partial cross-entropy.}
This kernel computes logits, selects the target logit, and emits block-wise log-sum-exp statistics for the cross-entropy loss:
\begin{align*}
\mZ &= \mA\mB, \\
\vz_{\mathrm{tgt}} &= \mZ[\vy], \\
\widehat{\vl}_{\mathrm{lse}} &= \operatorname{reduceTile}_{\log\sum\exp}(\mZ).
\end{align*}

\subsubsection{Forward-Pass Kernels}

The following kernels implement the reparameterized Transformer forward pass. They compose the basic epilogue primitives with RMSNorm scaling, residual updates, and partial reductions.

\textbf{Kernel 4: GEMM with residual, partial RMSNorm, and weight scaling.}
This kernel implements the first stage of the GEMM--Residual--RMSNorm--GEMM pattern. It forms the residual-updated activation, emits partial RMS statistics, and applies the RMSNorm weight:
\begin{align*}
\mD &= \mA\mB + \mC, \\
\widehat{\vr} &= \operatorname{reduceTile}_{\mathrm{cols}}(\mD \odot \mD), \\
\mO &= \mD \odot \boldsymbol{\gamma}.
\end{align*}

\textbf{Kernel 5: GEMM with RMSNorm scaling.}
This kernel consumes a precomputed row-wise normalization factor and applies it in the GEMM epilogue:
\begin{align*}
\mD &= \mA\mB, \\
\mO &= \mD \odot \vr.
\end{align*}

\textbf{Kernel 6: GEMM with RMSNorm and SwiGLU.}
This kernel composes row-wise RMSNorm scaling with SwiGLU, corresponding to the MLP gate/up projection:
\begin{align*}
\mD &= \mA\mB, \\
\mD^{\prime} &= \mD \odot \vr, \\
[\mG,\mU] &= \operatorname{interleavedSplit}(\mD^{\prime}), \\
\mO &= \operatorname{silu}(\mG) \odot \mU.
\end{align*}

\textbf{Kernel 7: GEMM with RMSNorm and RoPE.}
This kernel composes row-wise RMSNorm scaling with RoPE, corresponding to QKV projection followed by rotary positional embedding:
\begin{align*}
\mD &= \mA\mB, \\
\mD^{\prime} &= \mD \odot \vr, \\
\mO &= \operatorname{RoPE}(\mD^{\prime}).
\end{align*}

\textbf{Kernel 8: GEMM with RMSNorm and partial cross-entropy.}
This kernel adds row-wise RMSNorm scaling before target-logit selection and partial log-sum-exp reduction, corresponding to the language-modeling head:
\begin{align*}
\mZ &= (\mA\mB) \odot \vr, \\
\vz_{\mathrm{tgt}} &= \mZ[\vy], \\
\widehat{\vl}_{\mathrm{lse}} &= \operatorname{reduceTile}_{\log\sum\exp}(\mZ).
\end{align*}

\subsubsection{Backward-Pass Kernels}

Finally, we list the backward kernels. These kernels mirror the forward structure: each performs a GEMM, applies the local backward rule in the epilogue, and emits partial reductions needed by neighboring RMSNorm backward computations.

\textbf{Kernel 9: GEMM with residual and RMSNorm backward.}
This kernel implements the local part of RMSNorm backward. Let $\mC$ denote the RMSNorm input, $\vr$ the row-wise inverse RMS factor, $\boldsymbol{\gamma}$ the RMSNorm weight, and $\vz_{\Delta z}$ the row-wise RMSNorm backward statistic:
\begin{align*}
\mD &= \mA\mB^{\top}, \\
\mC_{\mathrm{norm}} &= \mC \odot \vr, \\
\mO_{\mathrm{out}}
&=
\mO_{\mathrm{in}}
+
\left(
\mD \odot \boldsymbol{\gamma}
-
\mC_{\mathrm{norm}} \odot \vz_{\Delta z}
\right)
\odot \vr, \\
\mC_{\mathrm{out}}
&=
\mC_{\mathrm{norm}} \odot \boldsymbol{\gamma}, \\
\widehat{\nabla_{\boldsymbol{\gamma}}\mathcal{L}}
&=
\operatorname{reduceTile}_{\mathrm{rows}}
\left(
\mD \odot \mC_{\mathrm{norm}}
\right).
\end{align*}

\textbf{Kernel 10: GEMM with SwiGLU backward.}
This kernel computes the backward pass of a fused SwiGLU epilogue and emits the row-wise statistic needed by the preceding RMSNorm backward. Let $\mZ$ be the saved interleaved pre-activation tensor:
\begin{align*}
\mD &= \mA\mB^{\top}, \\
[\mG,\mU] &= \operatorname{interleavedSplit}(\mZ), \\
\mO &= \operatorname{silu}(\mG) \odot \mU, \\
\nabla_{\mU}\mathcal{L}
&=
\mD \odot \operatorname{silu}(\mG), \\
\nabla_{\mG}\mathcal{L}
&=
\mD \odot \mU \odot
\left(
\sigma(\mG)
+
\operatorname{silu}(\mG) \odot (1-\sigma(\mG))
\right), \\
\nabla_{\mZ}\mathcal{L}
&=
\operatorname{interleavedConcat}
\left(
\nabla_{\mG}\mathcal{L},
\nabla_{\mU}\mathcal{L}
\right), \\
\widehat{\vz_{\Delta z}}
&=
\operatorname{reduceTile}_{\mathrm{cols}}
\left(
\mG \odot \nabla_{\mG}\mathcal{L}
+
\mU \odot \nabla_{\mU}\mathcal{L}
\right).
\end{align*}

\subsection{Setup Details}
\label{sec:setup-details}
Experiments are conducted using a single H100 GPU. We use the following package versions.
\begin{enumerate}
    \item PyTorch 2.10.0
    \item CuTeDSL 4.4.2
    \item Liger Kernels 0.8.0
    \item FlashInfer 0.6.10.post1
    \item QuACK Kernels 0.4.1
\end{enumerate}

\end{document}